\documentclass[11pt]{article}
\usepackage[top=1in, bottom=1.25in, left=1.8in, right=1.25in]{geometry}

\usepackage[utf8]{inputenc}
\usepackage[english]{babel}

\usepackage[breaklinks]{hyperref}
\hypersetup{%
  pdftitle={\@title},            %
  pdfauthor={\@author},          %
  unicode=true,           %
  colorlinks=true,       %
  linkcolor=blue,         %
  citecolor=blue,      %
  filecolor=green,        %
  urlcolor=blue           %
}

\usepackage[usenames,dvipsnames]{xcolor}

\usepackage[round]{natbib}

\usepackage{amsmath,amsfonts,amssymb,amsthm,amscd}
\newtheorem{example}{Example}

\usepackage{nameref}
\usepackage[nameinlink,capitalize]{cleveref}
\crefname{equation}{Equation}{Equations}
\crefname{assumption}{Assumption}{Assumptions}
\creflabelformat{enumi}{(#2#1#3)}

\usepackage{autonum}

\usepackage{comment,url,algorithm,algorithmic,graphicx,subcaption,relsize}
\usepackage{amssymb,amsfonts,amsmath,amsthm,amscd,dsfont,mathrsfs,mathtools,nicefrac}
\usepackage{float,psfrag,epsfig,color,xcolor,url,hyperref}
\usepackage{epstopdf,bbm,mathtools,enumitem}

\let\originalleft\left
\let\originalright\right
\renewcommand{\left}{\mathopen{}\mathclose\bgroup\originalleft}
\renewcommand{\right}{\aftergroup\egroup\originalright}

\def\reals{\mathbb{R}} %
\def\<{\left\langle} %
\def\>{\right\rangle}

\def\norm#1{\left\|{#1}\right\|} %
\newcommand{\fronorm}[1]{\norm{#1}_{\text{F}}} %

\def\E{\mathbb{E}} %

\DeclareMathOperator{\Tr}{Tr} %
\def\Var{\mathrm{Var}} %

\newcommand%
{\HypGeo}{\textnormal{HypGeo}}

\ifdefined\nonewproofenvironments\else
\ifdefined\ispres\else
\newtheorem{theorem}{Theorem}
\newtheorem{lemma}[theorem]{Lemma}

\newtheorem{definition}[theorem]{Definition}

\renewenvironment{proof}{\noindent\textbf{Proof.}\hspace*{.3em}}{\qed\\}
\newenvironment{proof-sketch}{\noindent\textbf{Proof Sketch}
  \hspace*{1em}}{\qed\bigskip\\}
\newenvironment{proof-idea}{\noindent\textbf{Proof Idea}
  \hspace*{1em}}{\qed\bigskip\\}
\newenvironment{proof-of-lemma}[1][{}]{\noindent\textbf{Proof of Lemma {#1}}
  \hspace*{1em}}{\qed\\}
\newenvironment{proof-of-theorem}[1][{}]{\noindent\textbf{Proof of Theorem {#1}}
  \hspace*{1em}}{\qed\\}
\newenvironment{proof-attempt}{\noindent\textbf{Proof Attempt}
  \hspace*{1em}}{\qed\bigskip\\}

\newenvironment{remark}{\noindent\textbf{Remark.}
  \hspace*{0em}}{\smallskip}%

\fi

\newtheorem{proposition}[theorem]{Proposition}

\newtheorem{assumption}{Assumption}
\fi
\makeatletter
\@addtoreset{equation}{section}
\makeatother

\hypersetup{
  colorlinks,
  linkcolor={red!50!black},
  citecolor={blue!50!black},
  urlcolor={blue!80!black}
}

\global\long\def\Nats{\mathbb{N}}
\global\long\def\Reals{\mathbb{R}}

\newcommand{\PP}{\mathbb{P}}
\newcommand{\EE}{\mathbb{E}}

\usepackage{mathtools}

\providecommand\given{} %
\newcommand\SetSymbol[1][]{
  \nonscript\,#1:\nonscript\,\mathopen{}\allowbreak}
\DeclarePairedDelimiterX\Set[1]{\lbrace}{\rbrace}%
{ \renewcommand\given{\SetSymbol[]} #1 }

\usepackage{bm}

\usepackage{sectsty}
\allsectionsfont{\sffamily}

\usepackage{titling}

\usepackage{authblk}
\title{Risk of the Least Squares Minimum Norm Estimator under the Spike
  Covariance Model}

\date{}
\author[,1]{Yasaman Mahdaviyeh\thanks{Corresponding author: \texttt{yasamanmdv@cs.toronto.edu}}}
\author[2]{Zacharie Naulet}
 
\affil[1]{University of Toronto\protect\\Department of Computer
  Science,\protect\\ 
  Vector Institute \protect \\ Toronto, ON, Canada \vspace{1em}}

\affil[2]{Université Paris-Saclay\protect\\CNRS, Laboratoire de mathématiques d'Orsay\protect\\ 91405, Orsay, France.}

\begin{document}

\maketitle

\begin{abstract}%
	We study risk of the minimum norm linear least squares estimator in when the number
	of parameters $d$ depends on $n$, and $\frac{d}{n} \rightarrow \infty$. 
	We assume that data has an underlying low rank structure by restricting ourselves to spike covariance matrices, where a
	fixed finite number of eigenvalues grow with $n$ and are much larger than the
	rest of the eigenvalues, which are (asymptotically) in the same order. We show
	that in this setting risk of minimum norm least squares estimator vanishes in compare to risk of the null estimator. We give asymptotic and non 
	asymptotic upper bounds for this risk, and also
	leverage the assumption of spike model to give an analysis of the bias that leads to tighter bounds in compare to previous works.
\end{abstract}

\section{Introduction}

\label{sec:introduction-1}

One of the recent approaches to explain good performance of neural networks
has focused on their ability to fit training data perfectly (interpolate)
without over-fitting. It has been shown that this property is not unique to neural nets,
and that simpler class of models such as kernel regression could exhibit
this behaviour, too
\citet{belk18,LR18}.
One of the simpler models where
interpolation has been studied recently is the least squares solution for linear
regression. In this case, interpolation is only guaranteed to happen in the high
dimensional setting where the number of parameters $d$ exceeds 
the number of samples $n$;
therefore, the least squares solution is not necessarily unique. However, the \textit{minimum norm least squares} (MNLS) solution is unique,  and can be written
in closed form. Also, if we minimize the squared loss using gradient descent with initial parameters
set to zero, we recover the minimum norm solution \citep{HMRT19}. This has, at least
partially, motivated several works that study the risk of minimum norm least
squares estimator for linear regression. 

We study the risk of MNLS Estimator
under the spike covariance model of \cite{johnstone2001},
where few population eigenvalues are much larger than the rest.
We expect spike covariance models to represent underlying low dimensional structure of the data, or fast decay of the eigenspectrum
of the sample covariance or Gram matrix that is observed in some
datasets, see for example \citet{LR18} and  \citet{JM09} for plots of
eigenspectrums of an image data set and microarray data respectively. 
Also, a few recent works have shown that decay of eigenvalues plays
an important role in interpolation without over-fitting in case of 
linear and kernel regression \citep{BLT19, LR18}. These factors have 
inspired us to study the MNLS estimator under spike covariance model.

We are especially interested in the \textit{High Dimensional Low Sample Size}
(HDLSS) regime introduced in
\citet{hall2005geometric} and later studied thoroughly in
\citet{ahn2007high,JM09,shen2013surprising,SSZM16,SSM16} where 
$d/n \rightarrow \infty$. This setting fits
scenarios where the amount of information collected about individuals grows much
faster than the number of individuals, as in collection of genetic data, for example \citep{JM09}. 
Combined with spike covariance model with diverging spiked eigenvalues,
we are in a setting where even though increasingly more features are collected, 
the new features are highly correlated with existing ones, and do not add much 
new information, but rather reinforce existing knowledge. It is interesting 
to note that in this setting, intuitively, roles of $n$ and $d$ are almost swapped, since the large number of highly correlated features has the redundancy effect of large $n$ and small $n$ with diverging variance plays 
the role of small $d$. This interpretation of this phenomenon can be 
even seen in analysis of the covariance matrices in high dimensions when 
the analysis is done through the Gram matrix \citep{HMRT19, wang2017}  

We rely on the characterization of limits of the eigenvalues and eigenvectors of
sample covariance matrix given in this regime in \citet{SSM16, KL16} and other
similar works to give an asymptotic bound on the risk which vanishes relative to
asymptotic risk of the null estimator. Under less restrictive assumptions on the
covariance matrix but stronger assumptions on the distribution of data, we also
give a high probability bound on the predictive risk that depends on spectral
gaps of the covariance matrix.  Our analysis of the bias is novel and leads to a
tighter bound on the bias in the HDLSS regime in compare to other related
work. We show that consistent estimation of the spiked eigenvalues and
eigenvectors leads to small risk, and thus highlight the connection between
\textit{Principal Component Analysis} (PCA) and linear regression.

This paper is organized as follows. First in \cref{setup,sec:spike-model}, we
introduce the model and the main assumptions. In \cref{related_work}, we review
current related literature, and we establish the main technical notations in
\cref{sec:notations}. We give the main theorems in \cref{main_results}: the
first is an almost-sure convergence of the risk of the MNLS estimator, given in
\cref{sec:minim-norm-estim}, while the second is a non-asymptotic bound for the
risk, given in \cref{sec:non-asympt-bounds}. Finally, we discuss the results in
\cref{sec:discussion}.  The proofs for everything related to the MNLS estimator
are postponed to \cref{sec:proofs} , and to \cref{sec:asympt-sample-covar} for
existing results on estimation of spike covariance matrices.

\subsection{Set up and assumptions}
\label{setup}

Given $n$ independent, identically distributed  pairs of data points and labels $(X_i, Y_i) \in \reals^d \times \reals$,
we assume the following model
\begin{align}
Y_i = \theta^T X_i + \xi_i, \;\; i = 1, \dots, n,
\end{align} 
where $\xi_i$ are mean $0$, \textit{independent and identically distributed}
(i.i.d.) noise, with variance $\EE[\xi_i^2] = \sigma^2$. The goal is to recover
the parameter $\theta$, or rather to have accurate predictions of future
observations, based upon the knowledge of $(X_1,Y_1),\dots,(X_n,Y_n)$. It is
helpful to collect the covariates $X_i$ into the $n \times p$ design matrix
$\bm{X}$ whose rows are $X_1,\dots,X_n$. Similarly, we collect the responses
inside the vector $\bm{Y} = (Y_1,\dots,Y_n)$. Here we are interested in the
\textit{Minimum Norm Least Squares} (MNLS) estimator, defined as
\begin{align}
\label{min_norm_est}
\hat{\theta} = \arg \min_{\theta}
\norm{\theta} \quad \mathrm{s.t} \quad \bm{X}\theta = \bm{Y}.
\end{align}
This estimator can be written in closed form as%
\begin{align}
\hat{\theta} = (\bm{X}^T\bm{X})^\dagger \bm{X}^T \bm{Y},
\end{align}
where $(\bm{X}^T\bm{X})^{\dagger}$ denotes the Moore-Penrose pseudo-inverse of
$\bm{X}^T\bm{X}$. This is quite well known, see for instance
\citet{penrose1956best}. As we will see, this notion of risk depends heavily on
the $d\times d$ random matrix $\hat{\Sigma} \coloneqq n^{-1}\bm{X}^T\bm{X}$,
that is, the sample covariance matrix. We let $\Sigma \coloneqq \EE[\hat{\Sigma}]$
denote the corresponding population covariance matrix. We write
$\hat{U}\hat{\Lambda}\hat{U}^T \coloneqq \hat{\Sigma}$ the singular value
decomposition of $\hat{\Sigma}$, where
$\hat{\Lambda} = \mathrm{diag}(\hat{\lambda}_1,\dots,\hat{\lambda}_d)$ are the
singular values of $\hat{\Sigma}$ sorted in non-increasing order, \textit{i.e.}
$\hat{\lambda}_1 \geq \hat{\lambda}_2 \geq \dots \hat{\lambda}_d$, and
$\hat{U}= (\hat{u}_1,\dots,\hat{u}_d)$. Under fairly weak assumptions on the
distribution of $\bm{X}$, it is immediately seen that $\hat{\Sigma}$ has rank
equal to $n$, thereby $\hat{\lambda}_{n+1},\dots,\hat{\lambda}_d =
0$. Similarly, we let $U\Lambda U^T \coloneqq \Sigma$ denote the singular value
decomposition of $\Sigma$, where
$\Lambda = \mathrm{diag}(\lambda_1,\dots,\lambda_d)$ are the singular values of
$\Sigma$ sorted in non-increasing order, and $U = (u_1,\dots,u_d)$ 
the eigenvectors. Here we
assume the following on the distribution of $\bm{X}$.

\begin{assumption}[Distribution of $X_i$, weak]
	\label{ass:iid_comp}
	As in \citet[Assumption 1]{SSM16}, we assume that $X_1,\dots,X_n$ are
	independent, identically distributed (iid) and the random vectors
	$Z_i \coloneqq \Lambda^{-1/2}U^TX_i$ have iid entries with zero mean, unit
	variance, and finite fourth moments, \textit{i.e.}
	$Z_i = (Z_{i,1},\dots,Z_{i,d})$ where the variables $(Z_{i,k})$ are iid with
	$\EE[Z_{i,k}] = 0$, $\EE[Z_{i,k}^2] = 1$, and $\EE[Z_{i,k}^4] < \infty$.
\end{assumption}

We use \Cref{ass:iid_comp} and \Cref{ass:1}, which is introduced in next section  to
establish almost-sure bounds for the MNLS estimator  in \cref{sec:minim-norm-estim}. Non-asymptotic upper bounds
on the risk are obtained in \cref{sec:non-asympt-bounds}, under the following
additional structure on the distribution of $\bm{X}$.

\begin{assumption}[Distribution of $X_i$, strong]
	\label{ass:iid_comp_strong}
	The \cref{ass:iid_comp} holds, and in addition the random variables
	$(Z_{i,k})$ are sub-gaussian. That is, there exists $\nu > 0$ such that
	$\log \EE[e^{\lambda Z_{i,k}}] \leq \frac{\lambda^2 \nu}{2}$ for all
	$\lambda \in \Reals$.
\end{assumption}

\subsection{Spike model and HDLSS regime}
\label{sec:spike-model}

The spike model was first studied in \cite{johnstone2001}, where a fixed number
of eigenvalues are greater than one, and the rest are one. It was motivated by
some cases such as functional data analysis \citep{johnstone2001}, and financial
data \citep{BAIK20061382}, where empirically, first few eigenvalues of sample
covariance matrix are much larger than the rest. A spike covariance matrix can
also be thought of as a perturbation of a low rank matrix, that is
$\Sigma = M + \delta I$, where $M$ is a rank $m \ll n$ matrix with large
eigenvalues.

In the context of PCA,
\citet{hall2005geometric,ahn2007high,JM09,shen2013surprising,SSZM16,SSM16}
introduce the so-called HDLSS regime as a realistic model for data, where they assume that $d \equiv d^{(n)}$ with
$d^{(n)}/n \to \infty$, and that the eigenvalues of
$\Sigma \equiv \Sigma^{(n)} = \sum_{j=1}^d\lambda_j^{(n)}u_j^{(n)} u_j^{(n)T}$
are such that a few of them are very large and dominate the rest of the
eigenvalues. In particular, they study the spike covariance model with diverging spiked eigenvalues, where the amount of \textit{signal}
increases as $n \to \infty$, in the sense that for the first $\bar{m}$
eigenvalues we have $\lambda_1^{(n)} \to \infty$, $\dots$,
$\lambda_{\bar{m}}^{(n)} \to \infty$. Formally, the assumption is the following.
\begin{assumption}[HDLSS]
	\label{ass:1}
	$\lim_{n\to \infty}d^{(n)}/n = \infty$, and there exists $\bar{m} \in \Nats$
	and $c_1,c_2 > 0$ such that the sequence of eigenvalues
	$\lambda_1^{(n)}> \lambda_2^{(n)} >\dots >\lambda_d^{(n)}$ satisfies
	$\lim_{n\to \infty} n\lambda_{\bar{m}}^{(n)}/d^{(n)} = \infty$,
	$\lim_{n\to \infty}\lambda_j^{(n)}/\lambda^{(n)}_{j+1} > 0$ for
	$j=1,\dots,\bar{m}$, while $\lim_{n\to \infty}\lambda_{\bar{m}+1}^{(n)} = c_1$
	and $\lim_{n\to \infty}\lambda^{(n)}_d = c_2$.
\end{assumption}
\cite{SSM16,SSZM16} show that the first $\bar{m}$ samples eigenvalues
$\hat{\lambda}_1,\dots,\hat{\lambda}_{\bar{m}}$ of $\hat{\Sigma}$ are consistent
for $\lambda_1^{(n)},\dots,\lambda^{(n)}_{\bar{m}}$, in the sense that
$\lim_{n\to\infty}\max_{j=1,\dots,\bar{m}}\hat{\lambda}_j/\lambda_j^{(n)} = 1$,
and also $\hat{u}_1,\dots,\hat{u}_{\bar{m}}$ are consistent for
$u_1^{(n)},\dots,u^{(n)}_{\bar{m}}$. We mainly use these results to bound
the risk of the MNLS estimator.

In the sequel, and thus to avoid the heavy notations, we drop the superscript
$(n)$ and we write $d$ for $d^{(n)}$, $\Sigma$ for $\Sigma^{(n)}$, $\lambda_j$
for $\lambda_j^{(n)}$, and $u_j$ for $u_j^{(n)}$, while keeping in mind that
those are considered to be sequences indexed by $n$. The same goes
for $\theta^{(n)} \equiv \theta$.

The spike covariance model might seem like a very specific case to study. 
We borrow here \citet[Example~4.1]{JM09}, which shows that the \cref{ass:1} on
$\Sigma$ can be encountered even by really simple covariance designs. We refer to the
aforementioned papers for more examples and thorough discussions.

\begin{example}
	\label{ex:11}
	In the case where $\Sigma_{i,i} = 1$ for $i=1,\dots,d$ and $\Sigma_{i,j} = a$
	for $i\ne j$, then the first eigenvalue of $\Sigma$ is
	$\lambda_1 = 1 + (d-1)a$, while $\lambda_2,\dots,\lambda_d = 1 - a$. Then we
	have $\frac{d}{n\lambda_1} \to 0$, $\frac{d}{n \lambda_2} \to \infty$,
	$\lambda_1 = \Theta(d)$.%
\end{example}

\subsection{Related work}
\label{related_work}

Here, we give a short overview of existing works that give bounds on the risk of
MNLS estimator in a high dimensional regime (where interpolation could
happen).%

\medskip%
\citet{Belkin2019TwoMO} study mean squared error in a finite sample and dimension
setting with isotropic Gaussian data. Under fairly general setting,
\citet{HMRT19} give asymptotic risk bounds when
$\frac{n}{d} \rightarrow \alpha \in (0, \infty)$ for general covariance
matrices, assuming their operator norm is bounded.  Their bound in the general
case depends on some results in random matrix theory, making it difficult to
interpret. This is expected if there are no restrictions on structure of the
covariance matrix. They also give more explicit bounds for some special
covariance matrices, including an equicorrelated covariance matrix, which is a
single spike model. For this covariance matrix, their results would be the same
as ours if we take limit of $\frac{d}{n} \rightarrow \infty$. They compare 
the risk bounds they get to \textit{signal to noise ratio} (SNR) which turns 
out to be equal to risk of the null estimator within constant factors. Likewise, 
when we give the
asymptotic bound, we look at the ratio of bias to risk of the null estimator.

\citet{BLT19} study this (unnormalized) risk for Gaussian data in an infinite
dimensional Hilbert space, with finite samples, and give conditions on the
covariance matrix such that the risk is small with high probability. They call
covariance matrices that meet these conditions benign. Intuitively, for a
covariance matrix to be benign, the eigenvalues must decay but not too
fast. \citet{BLT19} break down the spectrum of covariance matrix into the larger eigenvalues and the tail, and their bound depends on where the spectrum is partitioned. 
Spike covariance matrices do have some properties of benign covariance matrices, even though their eigenvalues don't decay to zero. Applying their partitioning of the spectrum to the spike covariance matrices considered
here, we get spiked and non spiked eigenvalues. Indeed, their bounds
can still hold under our setting; this will be discussed more thoroughly in
\cref{sec:more-comparison-with}.

Finally, we mention that the main novelty and essential difference with
\citet{HMRT19,BLT19} resides in the way we analyze the bias the MNLS
estimator. For instance, in \cite{BLT19}, bias is bounded by operator norm of
the difference between sample and population covariance matrices, which is in
turn bounded in probability using the bounds in \citet{KL16}. Here, we leverage
the extra structure we assume on the covariance matrix to perform a finer
analysis of the bias. In particular, we rely on findings of \citep{SSM16}
that the first $m$ eigenvalues and eigenvectors of $\hat{\Sigma}$ are
asymptotically consistent for their corresponding population counterparts. We
emphasize that the bias depends on which subspace is not spanned by data, that is, the null space of $\bm{X}$, and the norm of the projection of the
parameter
$\theta$ into the null space of $\bm{X}$. Sample covariance eigenvectors
corresponding to nonzero eigenvalues form a basis for the row space of $\bm{X}$,
thus characterizing them enables us to examine bias closely.

\subsection{Notations}
\label{sec:notations}

We assume that all random variables are defined on a common probability space
$(\Omega, \mathcal{F}, \PP)$. We write expectations under $\PP$ as $\EE$. The
symbols $\EE_{\theta}[f(\bm{Y}) \mid \bm{X}]$ means that $\bm{Y}$ is assumed to
be $\bm{X}\theta + \bm{\xi}$, where $\bm{\xi} \sim N(0,\sigma^2I)$. For a matrix
$A$, we let $\sigma_{\min}(A)$ and $\sigma_{\max}(A)$ denote the smallest and
largest singular values of $A$ respectively. We also write $\sigma_j(A)$ the
$j$-th singular value of $A$, orderered such that
$\sigma_1(A) \geq \sigma_2(A) \geq \dots$. We denote the trace operator by
$\Tr$. We rewrite the covariance matrix as $\Sigma = \sum_{j=1}^d \lambda_j P_j$
where $P_j$ is the projection operator onto the $j$-th eigenspace of
$\Sigma$. Similarly we let
$\hat{\Sigma} = \sum_{j=1}^n \hat{\lambda}_j \hat{P}_j$.

\section{Main results}\label{main_results}

\subsection{Almost-sure bounds on risk of the MNLS estimator}
\label{sec:minim-norm-estim}

Let $X_{new}$ be a new sample from the same distribution as 
$X_1, \dots , X_n$. The expected error (at $\theta$) of the MNLS estimator can be decomposed into
the two terms%
\begin{equation}
\label{decomposition}
R_{\bm{X}}(\hat{\theta},\theta) \coloneqq%
\E_{\theta} [(X_{new}^T \theta - X_{new}^T \hat{\theta})^2 \mid \bm{X}] %
=
\theta^T (I - \hat{\Sigma}^\dagger \hat{\Sigma}) 
\Sigma
(I - \hat{\Sigma}^\dagger \hat{\Sigma}) \theta + 
\frac {\sigma^2}{n} \Tr(
\hat{\Sigma}^\dagger
\Sigma
),
\end{equation}
which are called bias
$B_{\bm{X}}(\hat{\theta},\theta)^2 \coloneqq \theta^T(I - \hat{\Sigma}^\dagger
\hat{\Sigma}) \Sigma (I - \hat{\Sigma}^\dagger \hat{\Sigma}) \theta$ and
variance
$V_{\bm{X}}(\hat{\theta},\theta) \coloneqq
\frac{\sigma^2}{n}\Tr(\hat{\Sigma}\Sigma)$, respectively. This is a standard 
derivation given in \cref{decomposition_expanded} . Theorem below gives an upper bound on asymptotic risk of the
MNLS estimator \eqref{min_norm_est}. Note that bias is essentially variance of
(noiseless) response after $\theta$ and $\Sigma$ are projected into some
subspace.  Intuitively, \cref{upper_bound_theorem} below shows that if spike eigenvalues grow
fast enough, then asymptotically, we incur no bias in the subspace spanned by
the spike eigenvectors. More specifically, we consider the maximum risk of
$\hat{\theta}$ over the classes of parameters, for $m \in \Nats$,
$0 < \delta < 1$, and $L > 0$%
\begin{equation}
\label{eq:74}
\mathcal{A}(m,\delta)%
\coloneqq \Set*{\theta \in \Reals^d \given  %
	\textstyle\sum_{j=m+1}^d\|P_j\theta\|^2 \leq \delta \|\theta\|^2}.
\end{equation}

\begin{theorem}
	\label{upper_bound_theorem}
	Under \cref{ass:iid_comp,ass:1} for all $L > 0$, for all $0 \leq \delta < 1$,
	it holds almost-surely as $n\to \infty$,
	\begin{equation}
	\label{eq:51}
	\sup_{\theta \in \mathcal{A}(\bar{m},\delta)}\frac{B_{\bm{X}}(\hat{\theta},\theta)^2}{\Var(X_1^T\theta)}%
	=  O\Big( \frac{1}{n} \bigvee \frac{d}{n\lambda_1}\Big)%
	\times \Big\{ \delta + O\Big( \frac{1}{n} \bigvee \frac{d}{n\lambda_1}\Big)
	\Big\} = o(1),
	\end{equation}
	and, almost-surely as $n\to \infty$,
	\begin{equation}
	\label{eq:65}
	\sup_{\theta\in \Reals^d}V_{\bm{X}}(\hat{\theta},\theta)%
	\leq \frac{\sigma^2\bar{m}}{n}\Big\{1 + o(1) +
	O\Big(\sqrt{\frac{n\lambda_1}{d}\Big(1 \bigvee \frac{\lambda_1}{d}} \Big)
	\Big\}%
	+ O\Big( \frac{\sigma^2 n}{d} \Big).
	\end{equation}
	Consequently if $\lambda_1 = o\big(\frac{d\sqrt{n}}{\sigma^2}\big)$ and
	$\sigma^2= o\big(\frac{d}{n}\big)$, then
	$R_{\bm{X}}(\hat{\theta},\theta) = o(\Var(X_1^T\theta))$.
\end{theorem}

\begin{remark}
  \label{rmk:2}
  The previous theorem shows that in the HDLSS regime under fairly reasonable
  assumptions on $\theta$, the MNLS estimator always perform (asymptotically)
  better than the trivial null estimator, whose risk is given by
  $\Var(X_1^T\theta)$. In particular, for
  $R_{\bm{X}}(\hat{\theta},\theta)/\Var(X_1^T\theta)$ to vanish (analysis of the
  proof of \cref{upper_bound_theorem} show that this can vanish fast), it
  suffices that the parameter $\theta$ is sufficiently oriented along the
  directions of the spiked eigenvectors of $\Sigma$. In particular, it is enough
  to have $0 \leq \delta <1$, such that
  $\sum_{j=\bar{m}+1}^{d}\|P_{j}\theta\|^2 \leq \delta\|\theta\|^2$.
\end{remark}

\subsection{Non asymptotic bounds for the MNLS estimator}
\label{sec:non-asympt-bounds}

The asymptotic bound of the previous section tells us that under
\cref{ass:iid_comp,ass:1} the risk of the MNLS estimator vanishes in the limit,
but the bound is not very informative on what are the essential features of the
covariance matrix $\Sigma$ and $\theta$ that can make the risk small. In order
to get a better comprehension of the risk, we propose to investigate
non-asymptotic bounds.

In order to get non-asymptotic bounds on the risk of the MNLS estimator, we require finer characterization of the spectrum of $\Sigma$. Indeed, the key
assumption to understand the risk of the MNLS estimator is how the spectrum of
$\Sigma$ is spread, and especially the spectral gap between its eigenvalues. 
We now introduce the main definitions we need to characterize the
spectrum of $\Sigma$.

\begin{definition}[Spectral gap]
  \label{def:1}
  Let $G_j \coloneqq \lambda_j - \lambda_{j+1}$ denote the $j$-th spectral gap
  of $\Sigma$, and let $\bar{G}_1 \coloneqq G_1$,
  $\bar{G}_j = \min\{G_{j-1},G_j \}$ for $j\geq 2$. We also define the following
  global measure of spectral gap. For every $m \geq 1$ and every
  $\alpha \in \Reals$, we let
  $\mathcal{G}_m(\alpha) \coloneqq \sum_{j=1}^m(\lambda_j^{\alpha}/\bar{G}_j)$.
\end{definition}

Then we can establish the following non-asymptotic upper bound on the risk. Note
that the bounds on the next theorem are true under \cref{ass:iid_comp_strong}
only, but do not require \Cref{ass:1} to hold, as we discuss hereafter.

\begin{theorem}
	\label{thm:2}
	Let
	$\rho_n(m) \coloneqq n\big(\sqrt{\frac{d\lambda_{m+1}}{n\lambda_1}}\bigvee \frac{d
		\lambda_{m+1}}{n\lambda_1}\big)$. Then, there exists a universal constant $C
	> 0$ such that with $\PP$-probability at least $1 - e^{-t}$, for all $\theta
	\in \Reals^d$,
	\begin{multline}
	\label{eq:54}
	B_{\bm{X}}(\hat{\theta},\theta)%
	\leq 2\lambda_1\|\theta\|^2 \min_{m=1\dots,n}\Big\{
	\Big(\frac{\lambda_{m+1}}{\lambda_1} + C\lambda_1
	\mathcal{G}_m\Big(\frac{1}{2} \Big)^2\frac{m \vee \rho_n(m)^2 \vee t}{n}
	\Big)\\
	\times \Big( C\lambda_1^2 \mathcal{G}_m(0)^2 \frac{m \vee \rho_n(m)^2 \vee
		t}{n} + \frac{\sum_{j=m+1}^d\|P_j\theta\|^2}{\|\theta\|^2} \Big) \Big\},
	\end{multline}
	and, letting for simplicity $\alpha = C\sqrt{\frac{n\vee t}{d}}$,
	$\beta = C\sqrt{\frac{\bar{m}\vee t}{n}}$ and
	$\delta = C\sqrt{\frac{m\vee t}{n} } \bigvee \rho_n(m)$,
	\begin{multline}
	\label{eq:68}
	V_{\bm{X}}(\hat{\theta},\theta)%
	\leq \min_{\bar{m}=1,\dots,n}\min_{m=1,\dots,\bar{m}}\Big\{%
	\frac{\sigma^2}{n}\Big(1 + \frac{d \lambda_{\bar{m}+1}}{n\lambda_m}\Big(1 +
	\alpha\Big) + \beta\Big)\Big(2m + \lambda_1 \mathcal{G}_m(1) \delta\Big)\\
	+\frac{\sigma^2}{1 - \alpha}\Big( 2 \delta m \mathcal{G}_m(1) +
	\frac{n\lambda_{m+1}}{\lambda_1} \Big)\frac{\lambda_1}{d\lambda_d} \Big\}.
	\end{multline}
\end{theorem}

\begin{remark}
	\label{rmk:1}
	\Cref{thm:2} emphasizes that the HDLSS regime is only one idealized
	setting where the risk vanishes, and that \Cref{ass:1} is certainly not
	required for this purpose. In particular, the bound in \cref{thm:2} is valid
	regardless of any assumption on $\Sigma$ or $\theta$. Though difficult to
	read, a careful analysis of each term tells us that the fundamental condition
	to meet is to have a sufficiently fast decay of $m\mapsto \lambda_m/\lambda_1$
	as $m$ grows to reduce the bias, but not too fast so that the
	variance doesn't explode. Note that those conditions are reminiscent to \citet{BLT19}
	findings too.
\end{remark}

\begin{remark}
	\label{rmk:3}
	\cref{thm:2} also emphasizes that the more the parameter $\theta$ is
	aligned with the dominating eigenvectors of $\Sigma$ the
	smaller the bias will be. Note that this is not only to make the term
	$\sum_{j=m+1}^d\|P_j\theta\|^2$ small\footnote{Indeed, this is not even
		necessary to have a vanishing risk, though it certainly improves the bias.},
	but also more importantly, the faster $m \mapsto \sum_{j=m+1}^d\|P_j\theta\|^2$
	decays, the more $\lambda_1\|\theta\|^2 \approx \Var(X_1^T\theta)$, which
	is risk of the null estimator we want at least to beat. 
\end{remark}

\begin{remark}
	\label{rmk:4}
	Under \Cref{ass:iid_comp_strong} only, \textit{i.e.}, sub-gaussianity of
	the $X_i$'s, the setting we investigate is quite close to the one in
	\citet{BLT19}, where the authors study the same problem under the assumption
	that $X_i \sim N(0,\Sigma)$. Nevertheless, the bounds are quite different and
	it seems difficult to relate them. The setting in \citet{BLT19} is more
	general, as they make no assumption about the spectral gap of $\Sigma$. We
	believe that their bound can be better in situations where the spectral gap is
	small, as ours could deteriorate rapidly. We expect, however, our bound to be
	slightly better if the spectral gap gets larger. We discuss this point more
	thoroughly in \cref{sec:more-comparison-with}. Note that if the spectral gap
	is small while the eigenvalues can be grouped into small blocks, such that the blocks are sufficiently separated, then one can carry a similar analysis as
	ours too, using the same arguments as the ones in usual PCA literature
	\citep{KL16,SSM16}.
\end{remark}

\section{Discussion}
\label{sec:discussion}

\subsection{General discussion}
\label{sec:general-discussion}

Although the asymptotic analysis might seem odd at first, and in particular the
requirement that the risk vanish at infinity, this aims to provide some guidance
on the conditions under which interpolation can lead to reasonable
answers. Indeed, this is also coherent with \citet{BLT19} results, the
asymptotic analysis tells us that we can expect the risk to be small in
situations where a few eigenvalues dominate the others and the parameter is
relatively well aligned with the directions of eigenvectors corresponding to
dominating eigenvalues.

\subsection{The HDLSS regime and interpolation}
\label{sec:hdlss-regime}

The HDLSS regime is an idealization of the situation where a few eigenvalues of
the population covariance matrix dominates the rest of them. As we point out in
the \cref{thm:2}, the risk of the MNLS estimator can vanish in other situations
too, though the HDLSS regime is the prototypical example of sufficient
conditions where interpolation of the data and vanishing predictive risk can
coexist. Still, the requirement that
$\lambda_1,\dots,\lambda_{\bar{m}} \to \infty$ as $n,d\to \infty$ 
might seem very unrealistic at first, even having in mind the \cref{ex:11}.  Indeed, this is not as idealized as it seems and occurs when over time we collect highly
correlated features about individuals faster than we collect new individuals. In
this situation, even though $d/n \to \infty$, the ``effective'' number of
features about which we collect information remains small, as they are all
highly correlated. Then, the amount of information (\textit{i.e.} signal) in the
data indeed increases over time, which translates by saying that a certain
number of eigenvalues grow as $n,d \to \infty$ (corresponding to the directions
of the effective features). In other words, collecting large number of
correlated features, most of which are redundant,  increases the amount of
information contained in the data about the effective features. Interestingly,
interpolation and good prediction are not discordant in such a situation.

\subsection{What if  $d/n \lambda_1$ doesn't vanish?} \label{sec:slow_signal}

One might wonder how essential it is for $d/n \lambda_1 \to 0$ to get 
small risk in compare to the null estimator. Considering the minimax lower 
bounds for linear regression given in \citet{duchi2013}, along with results of \citet{wang2017} gives us some insight into this scenario. Under slightly stronger 
assumption on the distribution of data, \citet{wang2017} show that if for spiked 
eigenvalues indexed by $ 1\le j \le \bar{m}$ we have $d/n \lambda_j = c_j < \infty$, then 

\begin{equation}\label{eq:100}
\hat{\lambda}_j/\lambda_j = 1 + c_j + O_{\PP}\Big(\lambda_j^{-1}
\sqrt{\frac{d}{n}} \Big).
\end{equation}
On the other hand, the minimax lower bound given in \citet{duchi2013} is of the form
\begin{equation}
c \frac{d^2 \sigma^2}{\fronorm{X}^2}.
\end{equation}
Since $\fronorm{X}^2 = n \sum_{i=1}^n \hat{\lambda}_i$, 
\begin{equation}
c \frac{d^2 \sigma^2}{\fronorm{X}^2}
\ge 
c \frac{d^2 \sigma^2}{n^2 \hat{\lambda}_1}.
\end{equation}
Then using results of \citet{wang2017} in \eqref{eq:100} (while absorbing
constant terms in $c$), we get that
$ c \frac{d^2 \sigma^2}{n^2 \hat{\lambda}_1} \approx c \frac{d^2 \sigma^2}{n^2
  \lambda_1} \approx c \frac{d \sigma^2}{n}$, which diverges. We note that this
doesn't imply a a lower bound for the normalized risk since it doesn't depend on
$\norm{\theta}$, unless we assume that $\norm{\theta}$ remains bounded as
$n,d \to \infty$ to remove dependence of $\Var(X_1^T \theta)$ on
$\norm{\theta}$, then $\Var (X_1^T \theta) \approx \lambda_1$, which means that
even the normalized risk would be bounded away from zero.  Furthermore, this
minimax bound applies to the estimation risk, rather than the predictive risk
that we are considering here, though they can be in the same order under strong
but standard assumptions on design matrix.

It is also interesting that in this setting, limits of sample spiked and non spiked eigenvalues are the same asymptotically. For $ m+1 \le j \le n$, 
$\frac{n \hat{\lambda}_j}{d \lambda_j} \rightarrow 1$ \citep{SSZM16}. That
is, both spiked and non spiked eigenvalues of the sample covariance
matrix grow at the same rate of $\frac{d}{n}$, making them hard to
distinguish which also hints at why the risk might be large in this case.

\subsection{Comparison with \citet{BLT19}}
\label{sec:more-comparison-with}

As already discussed in \cref{related_work}, the main difference with
\citet{BLT19} resides in the way we analyse the bias of the MNLS. Indeed, since
we work in a more restricted setting, \textit{i.e.} the HDLSS regime, we take
benefit from the extra structure to improve on the bias. Indeed, inspection of
the proof of \citet[Lemma~8]{BLT19} shows that they bound the bias as,
\begin{equation}
\label{eq:60}
B^{BLT}_{\bm{X}}(\hat{\theta},\theta)^2 \leq \|\theta\|^2\Big\|\Sigma -
\frac{\bm{X}^T\bm{X}}{n} \Big\|.
\end{equation}
They further bound in probability the term $\|\Sigma - n^{-1}\bm{X}^T\bm{X}\|$
using the general results from \citet{KL16}, which requires the
\textit{effective rank}
$r(\Sigma) \coloneqq \frac{1}{\lambda_1}\sum_{k=1}^d\lambda_j$ to be a
$o(n)$. In the HDLSS regime, the effective rank is asymptotically equal to
$\bar{m} + n \rho_n(\bar{m})^2$, where where $\rho_n(m)$ is defined in
\cref{thm:2}. Hence the bound in \citet{BLT19} gives in the HDLSS regime of
\cref{ass:iid_comp_strong}, as $n,d\to \infty$,
\begin{equation}
\label{eq:62}
B^{BLT}_{\bm{X}}(\hat{\theta},\theta)^2 \leq \lambda_1 \|\theta\|^2 \times
O_{\PP}\Big(\frac{1}{\sqrt{n}} \bigvee  \rho_n(\bar{m}) \Big).
\end{equation}
In comparison, the bound in the \cref{upper_bound_theorem} can be seen to be in
the HDLSS regime, as $n,d \to \infty$
\begin{equation}
\label{eq:66}
B_{\bm{X}}(\hat{\theta},\theta)^2 \leq
\lambda_1\|\theta\|^2 \times O\Big(\frac{1}{n}\bigvee \rho_n(\bar{m})^2\Big)\times \Big\{O\Big(\frac{1}{n} \bigvee \rho_n(\bar{m})^2 \Big) +
\frac{\sum_{j=m+1}^d\|P_j\theta\|^2}{\|\theta\|^2} \Big\},
\end{equation}
almost-surely.
Hence, the bound in this paper is sharper by several order of magnitude for the
HDLSS regime (especially if $\theta$ has most of its mass on the dominating
eigenvalues directions), showing that in this regime there is an interest in
exploiting the consistency of $\hat{P}_1,\dots,\hat{P}_m$ for
$P_1,\dots,P_m$. As already mentioned, \citet{BLT19} results don't rely on
separation of eigenvalues in contrast to our bounds, and we expect their bound
to become better in situations where the spectrum of $\Sigma$ is not separated
enough.

Finally, our work complements \citet{BLT19} results by showing that not only
harmless interpolation in linear regression is possible in the large $d$ small
$n$ regime, but also the bias can be
significantly smaller than expected if $\Sigma$ is well-behaved and $\theta$ is
well-aligned.

\subsection{Further directions}
\label{sec:further-directions}

In Theorem \ref{upper_bound_theorem}, we showed that in the setting where signal
grows fast enough with $n$, which is when spiked eigenvectors can be estimated
consistently, normalized risk will vanish. We suspect that it is possible to give a non trivial
lower bound for bias in the scenario briefly discussed in \cref{sec:slow_signal}. That is, if we assume that for
spiked eigenvalues $1 \le j \le m$, $\frac{d(n)}{n \lambda_j} \rightarrow c_j$
where $0 < c_j <\infty$, and Gaussian data, we could use results of
~\cite{wang2017} to characterize sample spiked eigenvalues. However, to get a
non trivial lower bound, we also need to know more
about behaviour of
non spiked sample eigenvectors (especially their projection into the spiked eigenspace), which is not explored in spike PCA literature.

\section{Proofs}
\label{sec:proofs}

\subsection{Preliminaries}
\label{sec:preliminaries-1}

Here we prove simultaneously the \cref{upper_bound_theorem,thm:2}. Indeed, we
prove the theorems by establishing bounds on the bias in
\cref{sec:ub-bias-proof}, and the variance in
\cref{sec:ub-variancet-proof}. These bounds are not tied to the HDLSS scenario
and may hold in a more general setting. The bounds mostly depends on the
spectral gap of $\Sigma$, as defined in \cref{def:1}. In the asymptotic
viewpoint of \cref{ass:1,upper_bound_theorem}, however, the expression for the
spectral gap simplifies quite consequently in the limit, which we emphasize in
the next trivial proposition.
\begin{proposition}
  \label{pro:14}
  Under \cref{ass:1}, it holds $\lim_{n\to \infty} \lambda_j/\bar{G}_j
  \asymp 1$ for all $j=,\dots, \bar{m}$, and consequently $\lim_{n\to \infty}
  \lambda_1^{1 - \alpha}\mathcal{G}_{\bar{m}}(\alpha) \leq C\bar{m}$ for a
  universal constant $C > 0$.
\end{proposition}

\subsection{Upper bound on the bias of MNLS estimator}
\label{sec:ub-bias-proof}

We summarize in the statement of the next lemma the results of this
section. Then, the bounds for the bias in \cref{upper_bound_theorem,thm:2}
follows from both the bound in \cref{lem:1}, the results on the behaviour of
$P_j - \hat{P}_j$ in the HDLSS regime, which we recall in
\cref{sec:behav-eigen-proj}, see the \cref{lem:3}, and \cref{pro:14}. We
summarize the proofs of \cref{upper_bound_theorem,thm:2} in
\cref{sec:proof-cor:1}.

\begin{lemma}[Bias]
  \label{lem:1}
  For any $\theta$ and any $\bm{X}$, the following bound is true. For all
  $m =1,\dots,d$,
  \begin{equation}
    \label{eq:9}
    B_{\bm{X}}(\hat{\theta},\theta)^2%
    \leq 2\|\theta\|^2\Big(\lambda_{m+1} + \Big\|
    \sum_{j=1}^m \sqrt{\lambda_j}(\hat{P}_j - P_j)\Big\|^2\Big)\Big(%
    \Big\|\sum_{j=1}^m(\hat{P}_j - P_j) \Big\|^2 +
    \frac{\sum_{j=m+1}^d\|P_j\theta\|^2}{\|\theta\|^2}\Big).
  \end{equation}
  In particular, the \cref{lem:3} implies that the following bounds are true.
  \begin{enumerate}
    \item If \cref{ass:iid_comp,ass:1} are true. Then, as $n \to \infty$,
    almost-surely,
    \begin{equation}
      \label{eq:69}
      B_{\bm{X}}(\hat{\theta},\theta)^2%
      \leq 2\lambda_1\|\theta\|^2\Big( \frac{\lambda_{\bar{m}+1}}{\lambda_1} +
      O\Big(\frac{1}{n} \bigvee \frac{d}{n\lambda_1} \Big) \Big) \Big(
      O\Big(\frac{1}{n} \bigvee \frac{d}{n\lambda_1} \Big) +
      \frac{\sum_{j=m+1}^d\|P_j\theta\|^2}{\|\theta\|^2}\Big).
    \end{equation}
    Further, remark that under \cref{ass:1},
    $\frac{\lambda_{\bar{m}+1}}{\lambda_1} = O\big(\frac{n}{d} \frac{d}{n\lambda_1}\big) =
    o\big(\frac{d}{n\lambda_1}\big)$.
    \item Let define
    $\rho_n(m) = n \big(\sqrt{\frac{d\lambda_{m+1}}{n\lambda_1}} \bigvee
    \frac{d\lambda_{m+1}}{n\lambda_1} \big)$. If \cref{ass:iid_comp_strong} is
    true, then there is a universal constant $C > 0$ such that with
    $\PP$-probability at least $1 - e^{-t}$
    \begin{multline}
      \label{eq:70}
      B_{\bm{X}}(\hat{\theta},\theta)%
      \leq 2\lambda_1\|\theta\|^2 \min_{m=,1\dots,n}\Big\{
      \Big(\frac{\lambda_{m+1}}{\lambda_1} + C\lambda_1
      \mathcal{G}_m\Big(\frac{1}{2} \Big)^2\frac{m \vee \rho_n(m)^2 \vee t}{n}
      \Big)\\
      \times \Big( C\lambda_1^2 \mathcal{G}_m(0)^2 \frac{m \vee \rho_n(m)^2 \vee
      t}{n} + \frac{\sum_{j=m+1}^d\|P_j\theta\|^2}{\|\theta\|^2}  \Big) \Big\}.
    \end{multline}
  \end{enumerate}
\end{lemma}

We now prove the \cref{lem:1}. Remark that by linearity
$\EE_{\theta}[\hat{\theta} \mid \bm{X}] = (\bm{X}^T\bm{X})^{\dagger}\bm{X}^T
\bm{X}\theta$ and thus the bias can be rewritten as
\begin{align}
  \label{eq:11}
  B_{\bm{X}}(\hat{\theta},\theta)^2 =%
  \|\Sigma^{1/2}(\theta - \EE[\hat{\theta} \mid \bm{X}])\|^2%
  &= \theta^T(I - (\bm{X}^T\bm{X})^{\dagger}\bm{X}^T\bm{X})\Sigma (I -
    (\bm{X}^T\bm{X})^{\dagger}\bm{X}^T\bm{X})\theta.
\end{align}
We wish to understand
$\Sigma^{1/2}(I - (\bm{X}^T\bm{X})^{\dagger}\bm{X}^T\bm{X})\theta$. Remark that
$(\bm{X}^T\bm{X})^{\dagger}\bm{X}^T\bm{X} = \sum_{j=1}^n \hat{P}_j$ and that
$\Sigma^{1/2} = \sum_{j=1}^d \sqrt{\lambda_j}P_j$. So,
$\Sigma^{1/2}(\bm{X}^T\bm{X})^{\dagger}\bm{X}^T\bm{X}\theta =
\sum_{j=1}^d\sqrt{\lambda_j}P_j \sum_{k=1}^n \hat{P}_k \theta$. We decompose as
follows,
\begin{align}
  \label{eq:12}
  \Sigma^{1/2}(\bm{X}^T\bm{X})^{\dagger}\bm{X}^T\bm{X}\theta - \Sigma^{1/2}\theta
  &= \sum_{j=1}^m \sqrt{\lambda_j} P_j \Big( \sum_{k=1}^n\hat{P}_k - I\Big)\theta%
    + \sum_{j=m+1}^d \sqrt{\lambda_j} P_j \Big(\sum_{k=1}^n \hat{P}_k - I\Big)\theta.
\end{align}

\paragraph{Bound on the first term of the rhs of \cref{eq:12}}

We rewrite each of the $P_j$ as $P_j = \hat{P}_j + (P_j - \hat{P}_j)$, and thus
\begin{align}
  \label{eq:20}
  \sum_{j=1}^m \sqrt{\lambda_j} P_j \Big( \sum_{k=1}^n\hat{P}_k - I\Big)\theta%
  &= \sum_{j=1}^m \sqrt{\lambda_j}\hat{P}_j \sum_{k=1}^n\hat{P}_k \theta%
    + \sum_{j=1}^m\sqrt{\lambda_j}(P_j - \hat{P}_j)\sum_{k=1}^n \hat{P}_k\theta%
    - \sum_{j=1}^m \sqrt{\lambda_j}P_j\theta\\
  &=\sum_{j=1}^m \sqrt{\lambda_j}\hat{P}_j \theta%
    + \sum_{j=1}^m\sqrt{\lambda_j}(P_j - \hat{P}_j)\sum_{k=1}^n \hat{P}_k\theta%
    - \sum_{j=1}^m \sqrt{\lambda_j}P_j\theta\\
  &= \sum_{j=1}^m \sqrt{\lambda_j}(\hat{P}_j - P_j)\Big(I -
    \sum_{k=1}^n\hat{P}_k \Big)\theta.%
\end{align}
Thus, we obtain that,
\begin{equation}
  \label{eq:6}
  \Big\|\sum_{j=1}^m \sqrt{\lambda_j} P_j \Big( \sum_{k=1}^n\hat{P}_k -
  I\Big)\theta \Big\|%
  \leq \Big\|\Big(I - \sum_{k=1}^n\hat{P}_k\Big)\theta\Big\| \cdot \Big\| \sum_{j=1}^m
  \sqrt{\lambda_j}(P_j - \hat{P}_j) \Big\|.
\end{equation}
But,
\begin{align}
  \label{eq:57}
  \theta - \sum_{k=1}^n\hat{P}_k\theta%
  &= \theta - \sum_{k=1}^mP_k\theta%
    - \sum_{k=1}^m(\hat{P}_k - P_k)\theta%
    - \sum_{k=m+1}^n\hat{P}_k\theta\\
  &=\sum_{k=m+1}^d P_k \theta%
    - \sum_{k=1}^m(\hat{P}_k - P_k)\theta%
    - \sum_{\ell=1}^d \sum_{k=m+1}^n\hat{P}_kP_{\ell}\theta\\
  &=\sum_{k=m+1}^d P_k \theta%
    - \sum_{k=1}^m(\hat{P}_k - P_k)\theta%
    - \sum_{\ell=1}^m \sum_{k=m+1}^n\hat{P}_kP_{\ell}\theta%
    - \sum_{\ell=m+1}^d \sum_{k=m+1}^n\hat{P}_kP_{\ell}\theta\\
  &=\sum_{k=m+1}^d P_k \theta%
    - \sum_{k=1}^m(\hat{P}_k - P_k)\theta%
    - \sum_{\ell=1}^m \sum_{k=m+1}^n\hat{P}_k(P_{\ell} - \hat{P}_{\ell})\theta%
    - \sum_{\ell=m+1}^d \sum_{k=m+1}^n\hat{P}_kP_{\ell}\theta\\
  &= \Big(I - \sum_{\ell=m+1}^n\hat{P}_{\ell} \Big) \sum_{k=m+1}^d P_k \theta%
    - \Big(I - \sum_{\ell=m+1}^n\hat{P}_{\ell}\Big) \sum_{k=1}^m(\hat{P}_k - P_k)\theta.%
\end{align}
Therefore,
\begin{align}
  \label{eq:58}
  \Big\|\Big(I - \sum_{k=1}^n\hat{P}_k\Big)\theta\Big\|%
  &\leq \Big\|\sum_{k=m+1}^dP_k\theta\Big\|%
    + \Big\|\sum_{k=1}^m(\hat{P}_k - P_k)\theta\Big\|.%
\end{align}

\paragraph{Bound on the second term of the rhs of \cref{eq:12}}

For the sake of simplicity we let $\hat{Q} \coloneqq \sum_{k=1}^n \hat{P}_k
$. Then, using that $I = \sum_{\ell=1}^dP_{\ell}$ we rewrite,%
\begin{align}
  \label{eq:14old}
  \Big\|\sum_{j=m+1}^d \sqrt{\lambda_j} P_j \hat{Q} \theta\Big\|^2
  &= \sum_{j=m+1}^d\lambda_j \|P_j \hat{Q}\theta\|^2\\
  &=  \sum_{j=m+1}^d\lambda_j \Big\|P_j \hat{Q}\sum_{\ell=1}^mP_{\ell}\theta +
    P_j \hat{Q}\sum_{\ell=m+1}^dP_{\ell}\theta   \Big\|^2\\
  \label{eq:18}
  &\leq 2\sum_{j=m+1}^d \lambda_j \Big\|P_j \hat{Q}\sum_{\ell=1}^mP_{\ell}\theta
    \Big\|^2%
    + 2\sum_{j=m+1}^d \lambda_j \Big\|P_j \hat{Q}\sum_{\ell=m+1}^dP_{\ell}\theta
    \Big\|^2
\end{align}
Regarding the second term of the rhs of the last display,
\begin{align}
  \label{eq:17}
  \sum_{j=m+1}^d \lambda_j \Big\| P_j \hat{Q}\Big(\sum_{\ell=m+1}^d P_{\ell}
  \Big)\theta \Big\|^2%
  &\leq \lambda_{m+1} \sum_{j=m+1}^d \Big\| P_j \hat{Q}\Big(\sum_{\ell=m+1}^d P_{\ell}
    \Big)\theta \Big\|^2\\
  &= \lambda_{m+1} \Big\|\Big(\sum_{j=m+1}^d  P_j \Big)\hat{Q}\Big(\sum_{\ell=m+1}^d P_{\ell}
    \Big)\theta \Big\|^2\\
  &\leq \lambda_{m+1}\Big\|\Big( \sum_{\ell=m+1}^d P_{\ell}\Big) \theta
    \Big\|^2\\
  &= \lambda_{m+1}\sum_{\ell=m+1}^d \|P_{\ell}\theta\|^2.%
\end{align}
For the first term of the rhs of \cref{eq:18}, we can rewrite that
\begin{align}
  \label{eq:89}
  \hat{Q}P_{\ell}
  &= \hat{Q}P_{\ell} + \hat{Q}(P_{\ell} - \hat{P}_{\ell})\\%
  &= \hat{P}_{\ell} +  \hat{Q}(P_{\ell} - \hat{P}_{\ell})\\%
  &= P_{\ell} + (\hat{P}_{\ell} - P_{\ell}) + \hat{Q}(P_{\ell} -
    \hat{P}_{\ell})\\%
  &= P_{\ell} + (\hat{Q} - I)(P_{\ell} - \hat{P}_{\ell}),
\end{align}
and hence,
\begin{align}
  \label{eq:19}
  \sum_{j=m+1}^d \lambda_j \Big\|P_j\hat{Q}\sum_{\ell=1}^mP_{\ell}\theta
  \Big\|^2%
  &\leq \lambda_{m+1}\sum_{j=m+1}^d\Big\|\sum_{\ell=1}^m\Big(P_jP_{\ell} +
    P_j(\hat{Q}-I)(P_{\ell}-\hat{P}_{\ell})\Big)\theta \Big\|^2\\
  &= \lambda_{m+1} \sum_{j=m+1}^d\Big\|P_j(\hat{Q} - I)\sum_{\ell=1}^m(P_{\ell}
    - \hat{P}_{\ell})\theta \Big\|^2\\
  &\leq \lambda_{m+1}\Big\|\sum_{\ell=1}^m(P_{\ell} - \hat{P}_{\ell})\theta \Big\|^2.
\end{align}
Combining everything,
\begin{align}
  \label{eq:59}
  \Big\|\sum_{j=m+1}^d \sqrt{\lambda_j}P_j\Big(\sum_{k=1}^n \hat{P}_k -
  I\Big)\theta \Big\|^2%
  &\leq 2\Big\|\sum_{j=m+1}^d \sqrt{\lambda_j}P_j\Big(\sum_{k=1}^n
    \hat{P}_k\Big)\theta \Big\|^2%
    + 2 \Big\|\sum_{j=m+1}^d \sqrt{\lambda_j}P_j\theta \Big\|^2\\
  &\leq 2 \lambda_{m+1}\Big\|\sum_{\ell=1}^m(P_{\ell} - \hat{P}_{\ell})\theta
    \Big\|^2%
    + 2\lambda_{m+1}\sum_{j=m+1}^d\|P_j\theta\|^2.%
\end{align}

\subsection{Upper bound on the variance of MNLS estimator}
\label{sec:ub-variancet-proof}

We summarize in the statement of the next lemma the results of this
section. Then, the bounds for the variance in \cref{upper_bound_theorem,thm:2}
follows from both the bound in \cref{lem:4}, the results on behaviour of
$\hat{\lambda}_1,\dots,\hat{\lambda}_n$, which we recall in
\cref{sec:behav-sample-eigenv,sec:bahviour-non-spiked}, and the results on the
behaviour of $P_j - \hat{P}_j$ in the HDLSS regime, which we recall in
\cref{sec:behav-eigen-proj}.

\begin{lemma}
  \label{lem:4}
  For any $\theta$ and any $\bm{X}$, the following bound is true. For all
  $m=1,\dots,d$,
  \begin{multline}
    \label{eq:40}
    V_{\bm{X}}(\hat{\theta},\theta)%
    \leq \frac{\sigma^2m}{n}\Big(1 + \frac{\lambda_{m+1}}{\lambda_m}\Big)
    \max_{j=1,\dots,m}\frac{\lambda_j}{\hat{\lambda}_j}%
    +
    \frac{\sigma^2}{n}\max_{j=1,\dots,m}\frac{\lambda_j}{\hat{\lambda}_j}\Big(\sum_{j=1}^m
    \frac{1}{\lambda_j}\|\hat{P}_j - P_j\| \Big)
    \Big(\sum_{k=1}^m\lambda_k\Big)\\%
    + \frac{2\sigma^2m}{n\hat{\lambda}_n}\Big\|\sum_{k=1}^m\lambda_k(P_k -
    \hat{P}_k)\Big\|%
    + \frac{\sigma^2\lambda_{m+1}}{\hat{\lambda}_n}.
  \end{multline}
  In particular, the following bounds are true.
  \begin{enumerate}
    \item If \cref{ass:iid_comp,ass:1} are true. Then, as $n \to \infty$,
    almost-surely,
    \begin{equation}
      \label{eq:69}
      V_{\bm{X}}(\hat{\theta},\theta)%
      \leq \frac{\sigma^2\bar{m}}{n}\Big\{1 + o(1) +
      O\Big(\sqrt{\frac{n\lambda_1}{d}\Big(1 \bigvee \frac{\lambda_1}{d}} \Big)
      \Big\}%
      + O\Big( \frac{\sigma^2 n}{d} \Big).
    \end{equation}
    \item $\rho_n(m) = n \big(\sqrt{\frac{d\lambda_{m+1}}{n\lambda_1}} \bigvee
    \frac{d\lambda_{m+1}}{n\lambda_1} \big)$. If \cref{ass:iid_comp_strong} is
    true, then there is a universal constant $C > 0$ such that with
    $\PP$-probability at least $1 - e^{-t}$,
    \begin{multline}
      \label{eq:55}
      V_{\bm{X}}(\hat{\theta},\theta)%
      \leq \min_{\bar{m}=1,\dots,n}\min_{m=1,\dots,\bar{m}}\Big\{%
      \frac{\sigma^2}{n}\Big(1 + \frac{d
        \lambda_{\bar{m}+1}}{n\lambda_m}\Big(1 + \alpha\Big)  + \beta\Big)\Big(2m + \lambda_1
      \mathcal{G}_m(1) \delta\Big)
      \\
      +\frac{\sigma^2}{1 - \alpha}\Big(
      2 \delta m \mathcal{G}_m(1) + \frac{n\lambda_{m+1}}{\lambda_1} \Big)\frac{\lambda_1}{d\lambda_d}
      \Big\},
    \end{multline}
    where $\alpha = C\sqrt{\frac{n\vee t}{d}}$,
    $\beta = C\sqrt{\frac{\bar{m}\vee t}{n}}$ and
    $\delta = C\sqrt{\frac{m\vee t}{n} } \bigvee \rho_n(m)$.
  \end{enumerate}
\end{lemma}

In order to establish \cref{lem:4}, we recall that
$V_{\bm{X}}(\hat{\theta},\theta) =
\sigma^2\Tr(\hat{\Sigma}^{\dagger}\Sigma)/n$. Then, we decompose
$\hat{\Sigma}^{\dagger}\Sigma = \sum_{j=1}^n \frac{1}{\hat{\lambda}_j}\hat{P}_j
\sum_{k=1}^d\lambda_k P_k$ into four terms%
\begin{equation}
  \label{eq:23}
  \hat{\Sigma}^{\dagger}\Sigma%
  = \sum_{j=1}^m \sum_{k=1}^m \frac{\lambda_k}{\hat{\lambda}_j} \hat{P}_jP_k%
  + \sum_{j=1}^m \sum_{k=m+1}^d \frac{\lambda_k}{\hat{\lambda}_j}
  \hat{P}_jP_k%
  + \sum_{j=m+1}^n \sum_{k=1}^m \frac{\lambda_k}{\hat{\lambda}_j}
  \hat{P}_jP_k%
  + \sum_{j=m+1}^n \sum_{k=m+1}^d \frac{\lambda_k}{\hat{\lambda}_j} \hat{P}_jP_k.
\end{equation}
We bound the trace of each of the four terms above in the paragraphs
below. The final result follows by combining all these bounds.

\paragraph{Bound on the first term of \cref{eq:23}}

To bound the first term, we use that the projectors $\hat{P}_j$ are consistent
for $P_j$ in the operator norm when $j=1,\dots,m$. Then, we can rewrite
\begin{align}
  \label{eq:35}
  \sum_{j=1}^m\sum_{k=1}^m \frac{\lambda_k}{\hat{\lambda}_j}\hat{P}_jP_k%
  &= \sum_{j=1}^m\frac{\lambda_j}{\hat{\lambda}_j}P_j%
    + \sum_{j=1}^m\sum_{k=1}^m \frac{\lambda_k}{\hat{\lambda}_j}(\hat{P}_j - P_j)P_k.
\end{align}
Since $P_j$ has always rank $1$, by taking the trace of the previous expression
we obtain
\begin{align}
  \label{eq:36}
  \Tr\Big(\sum_{j=1}^m\sum_{k=1}^m \frac{\lambda_k}{\hat{\lambda}_j}\hat{P}_jP_k
  \Big)%
  &= \sum_{j=1}^m \frac{\lambda_j}{\hat{\lambda}_j}%
    + \Tr\Big(\sum_{j=1}^m \frac{1}{\hat{\lambda}_j}(\hat{P}_j -
    P_j)\sum_{k=1}^m\lambda_kP_k \Big)\\
  &\leq m \max_{j=1,\dots,m}\frac{\lambda_j}{\hat{\lambda}_j}%
    + \Tr\Big(\sum_{j=1}^m \frac{1}{\hat{\lambda}_j}(\hat{P}_j -
    P_j)\sum_{k=1}^m\lambda_kP_k \Big).
\end{align}
We bound the second term of the last display using von Naumann's trace
inequality. Indeed,
\begin{align}
  \label{eq:37}
  \Tr\Big(\sum_{j=1}^m\sum_{k=1}^m \frac{\lambda_k}{\hat{\lambda}_j}\hat{P}_jP_k
  \Big)%
  \leq  m \max_{j=1,\dots,m}\frac{\lambda_j}{\hat{\lambda}_j}%
  + \sum_{\ell=1}^d
  \sigma_{\ell}\Big(\sum_{j=1}^m\frac{1}{\hat{\lambda}_j}(\hat{P}_j - P_j)
  \Big)%
  \sigma_{\ell}\Big(\sum_{k=1}^m\lambda_k P_k \Big).
\end{align}
Now the matrix $\sum_{k=1}^m\lambda_kP_k$ has rank no more than $m$ so
$\sigma_{\ell}(\sum_{k=1}^m\lambda_kP_k) = 0$ if $\ell > m$, and
$\sigma_{\ell}(\sum_{k=1}^m\lambda_kP_k) = \lambda_{\ell}$ if
$1\leq \ell \leq m$. Further
$\max_{1\leq \ell\leq d}\sigma_{\ell}(\sum_{j=1}^m
\frac{1}{\hat{\lambda}_j}(\hat{P}_j - P_j)) = \sigma_1(\sum_{j=1}^m
\frac{1}{\hat{\lambda}_j}(\hat{P}_j - P_j)) \leq
\sum_{j=1}^m\frac{1}{\hat{\lambda}_j}\|\hat{P}_j - P_j\|$, and thus
\begin{align}
  \label{eq:38}
  \Tr\Big(\sum_{j=1}^m\sum_{k=1}^m \frac{\lambda_k}{\hat{\lambda}_j}\hat{P}_jP_k
  \Big)%
  &\leq m \max_{j=1,\dots,m}\frac{\lambda_j}{\hat{\lambda}_j}%
    + \Big(\sum_{j=1}^m \frac{1}{\hat{\lambda}_j}\|\hat{P}_j - P_j\| \Big)
    \Big(\sum_{k=1}^m\lambda_k\Big)\\
  &\leq m \max_{j=1,\dots,m}\frac{\lambda_j}{\hat{\lambda}_j}%
    + \max_{j=1,\dots,m}\frac{\lambda_j}{\hat{\lambda}_j}\Big(\sum_{j=1}^m
    \frac{1}{\lambda_j}\|\hat{P}_j - P_j\| \Big)
    \Big(\sum_{k=1}^m\lambda_k\Big).
\end{align}

\paragraph{Bound on the second term of \cref{eq:23}}

Using that $\Tr(\hat{P}_jP_k) = (\hat{u}_j^Tu_k)^2$, we indeed have
\begin{align}
  \label{eq:24}
  \Tr\Big(\sum_{j=1}^m \sum_{k=m+1} ^d \frac{\lambda_k}{\hat{\lambda}_j}
  \hat{P}_j P_k\Big)%
  &= \sum_{j=1}^m\sum_{k=m+1}^d \frac{\lambda_k}{\hat{\lambda}_j}
    (\hat{u}_j^Tu_j)^2\\
  &\leq \frac{\lambda_{m+1}}{\hat{\lambda}_m} \sum_{j=1}^m\sum_{k=m+1}^d
    (\hat{u}_j^Tu_j)^2\\
  &\leq \frac{\lambda_{m+1}}{\hat{\lambda}_m} \sum_{j=1}^m \|\hat{u}_j\|^2\\
  \label{eq:45}
  &= \frac{m \lambda_{m+1}}{\hat{\lambda}_m}.
\end{align}
Using the consistency of $\hat{P}_j$ for $P_j$ in the operator norm, we can get
another bound. Note that the second bound is not needed, as \cref{eq:45} is
already smaller than the dominating term of the variance, but we give it for
completeness.  Indeed, using von Neumann's trace inequality, we get
\begin{align}
  \label{eq:42}
  \Tr\Big(\sum_{j=1}^m \sum_{k=m+1} ^d \frac{\lambda_k}{\hat{\lambda}_j}
  \hat{P}_j P_k\Big)%
  & = \Tr\Big(\sum_{j=1}^m\frac{1}{\hat{\lambda}_j}(\hat{P}_j - P_j)
    \sum_{k=m+1} ^d \lambda_k P_k\Big)\\%
  &\leq \sum_{\ell=1}^d\sigma_{\ell}\Big(\sum_{j=1}^m
    \frac{1}{\hat{\lambda}_j}(\hat{P}_j - P_j) \Big)%
    \sigma_{\ell}\Big(\sum_{k=m+1}^d\lambda_k P_k \Big).
\end{align}
But the matrix $\sum_{j=1}^m\hat{\lambda}_j^{-1}(\hat{P}_j - P_j)$ has rank no
more than $2m$, so that
$\sigma_{\ell}(\sum_{j=1}^m\hat{\lambda}_j^{-1}(\hat{P}_j - P_j)) = 0$ for
$\ell > 2m$. Also,
$\max_{\ell =1,\dots,d}\sigma_{\ell}(\sum_{k=m+1}\lambda_kP_k) = \lambda_{m+1}$,
so that we have the bound
\begin{align}
  \label{eq:43}
  \Tr\Big(\sum_{j=1}^m \sum_{k=m+1} ^d \frac{\lambda_k}{\hat{\lambda}_j}
  \hat{P}_j P_k\Big)%
  &\leq
    2m\lambda_{m+1}\sigma_1\Big(\sum_{j=1}^m\frac{1}{\hat{\lambda}_j}(\hat{P}_j
    - P_j) \Big)\\
  &\leq 2m \lambda_{m+1}\sum_{j=1}^m\frac{1}{\hat{\lambda}_j}\|\hat{P}_j -
    P_j\|\\
  \label{eq:44}
  &\leq 2m\lambda_{m+1}
    \max_{j=1,\dots,m}\frac{\lambda_j}{\hat{\lambda}_j}\sum_{j=1}^m\frac{1}{\lambda_j}
    \|\hat{P}_j - P_j\|.
\end{align}
Combining \cref{eq:45,eq:44} it follows,%
\begin{equation}
  \label{eq:41}
  \Tr\Big(\sum_{j=1}^m \sum_{k=m+1} ^d
  \frac{\lambda_k}{\hat{\lambda}_j}\hat{P}_j P_k\Big)%
  \leq \frac{m\lambda_{m+1}}{\lambda_m}\max_{1\leq j \leq
    m}\frac{\lambda_j}{\hat{\lambda}_j}\Big(1 \bigwedge 2\sum_{j=1}^m
  \frac{1}{\lambda_j}\|\hat{P}_j - P_j\| \Big).
\end{equation}

\paragraph{Bound on the third term of \cref{eq:23}}

We use the argument that for $j=1,\dots,m$ the projectors $\hat{P}_j$ are
consistent for $P_j$, and thus in the limit the projector
$\sum_{j=m+1}^n \hat{P}_j$ is orthogonal to any $P_k$ for $k = 1,\dots,
m$. Indeed, we rewrite $P_k = P_k - \hat{P}_k + \hat{P}_k$ in the previous term
and use von Neumann's trace inequality to deduce that
\begin{align}
  \label{eq:26}
  \Tr\Big(\sum_{j=m+1}^n \sum_{k=1}^m \frac{\lambda_k}{\hat{\lambda}_j}
  \hat{P}_j P_k \Big)%
  &= \Tr\Big(\sum_{j=m+1}^n \frac{1}{\hat{\lambda}_j} \hat{P}_j
    \sum_{k=1}^m \lambda_k(P_k - \hat{P}_k) \Big)\\
  &\leq \sum_{\ell=1}^d \sigma_\ell\Big(\sum_{j=m+1}^n \frac{1}{\hat{\lambda}_j}\hat{P}_j
    \Big)\sigma_{\ell}\Big(\sum_{k=1}^m\lambda_k(P_k - \hat{P}_k) \Big).
\end{align}
Now $\max_{\ell=1}^d\sigma_{\ell}(\sum_{j=m+1}^n\hat{\lambda}_j^{-1}\hat{P}_j) =
\sigma_1(\sum_{j=m+1}^n \hat{\lambda}_j^{-1}
\hat{P}_j) = \hat{\lambda}_n^{-1}$, and the matrix $\sum_{k=1}^m\lambda_k(P_k - \hat{P}_k)$ has rank no more
than $2m$, from which we deduce that $\sigma_{\ell}(\sum_{k=1}^m\lambda_k(P_k -
\hat{P}_k)) = 0$ for $\ell > 2m$. Henceforth, 
\begin{align}
  \label{eq:28}
  \Tr\Big(\sum_{j=m+1}^n \sum_{k=1}^m \frac{\lambda_k}{\hat{\lambda}_j}
  \hat{P}_j P_k\Big)%
  &\leq%
    \frac{2m}{\hat{\lambda}_n} \sigma_1\Big(\sum_{k=1}^m \lambda_k(P_k - \hat{P}_k) \Big)\\
  &\leq \frac{2m}{\hat{\lambda}_n}\Big\|\sum_{k=1}^m\lambda_k(P_k - \hat{P}_k)\Big\|.
\end{align}

\paragraph{Bound on the last term of \cref{eq:23}}

To bound the last term, we again use that $\Tr(\hat{P}_jP_k) =
(\hat{u}_j^Tu_k)^2$ to deduce that
\begin{align}
  \label{eq:34}
  \Tr\Big(\sum_{j=m+1}^n \sum_{k=m+1}^d \frac{\lambda_k}{\hat{\lambda}_j}
  \hat{P}_j P_k \Big)%
  &=\sum_{j=m+1}^n\sum_{k=m+1}^d
    \frac{\lambda_k}{\hat{\lambda}_j}\Tr(\hat{P}_jP_k)\\
  &\leq
    \frac{\lambda_{m+1}}{\hat{\lambda}_n}\sum_{j=m+1}^n\sum_{k=m+1}^d(\hat{u}_j^Tu_k)^2\\
  &\leq\frac{\lambda_{m+1}}{\hat{\lambda}_n}\sum_{j=m+1}^n\sum_{k=m+1}^d(\hat{u}_j^Tu_k)^2\\
  &\leq \frac{\lambda_{m+1}}{\hat{\lambda}_n}\sum_{j=m+1}^n \|\hat{u}_j\|^2\\
  &\leq \frac{n\lambda_{m+1}}{\hat{\lambda}_n}.
\end{align}

\subsection{Summary of the proof of \cref{upper_bound_theorem,thm:2}}
\label{sec:proof-cor:1}

The proofs are an immediate consequence of \cref{lem:1,lem:4}, the only thing
remaining to show is to relate $\Var(X_1^T\theta)$ to $\lambda_1\|\theta\|^2$
when $\theta \in \mathcal{A}(\bar{m},L,\delta)$. But, we have for any $\theta \in \mathcal{A}(m,L,\delta)$
\begin{align}
  \label{eq:56}
  \frac{\lambda_1\|\theta\|^2}{\Var(X_1^T\theta)}%
  &= \frac{\lambda_1\|\theta\|^2}{\sum_{j=1}^d\lambda_j \|P_j\theta\|^2}%
  \leq \frac{\lambda_1\|\theta\|^2}{\sum_{j=1}^{m}\lambda_j\|P_j\theta\|^2}%
  \leq \frac{\lambda_1\|\theta\|^2}{\lambda_{m}\sum_{j=1}^{m}\|P_j\theta\|^2}\\
  &\leq \frac{\lambda_1\|\theta\|^2}{\lambda_{m}\big(\sum_{j=1}^d\|P_j\theta\|^2 -
    \sum_{j=m+1}^d\|P_j\theta\|^2\big)}%
  \leq \frac{\lambda_1}{\lambda_{m}} \frac{1}{1-\delta}.
\end{align}
Remark that we also always have $\Var(X_1^T\theta) \leq \lambda_1\|\theta\|^2$,
and hence
$ \frac{\lambda_m}{\lambda_1}(1-\delta) \lambda_1\|\theta\|^2 \leq
\Var(X_1^T\theta) \leq \lambda_1\|\theta\|^2$ for every
$\theta \in \mathcal{A}(m,\delta,L)$. For the \cref{upper_bound_theorem}, simply
remark that $\lambda_1 \asymp \lambda_{\bar{m}}$ and thus $\Var(X_1^T\theta)
\asymp \lambda_1\|\theta\|^2$ under the \cref{ass:1}.

\section{Acknowledgements}

YM is funded by NSERC and the Vector Institute. Part of this work has been done when
ZN was a postdoctoral fellow at the University of Toronto and at the Vector Institute.
ZN's work was supported by U.S. Air Force Office of Scientific Research grant
\#FA9550-15-1-0074.

\bibliography{refs}

\begin{thebibliography}{20}
\providecommand{\natexlab}[1]{#1}
\providecommand{\url}[1]{\texttt{#1}}
\expandafter\ifx\csname urlstyle\endcsname\relax
  \providecommand{\doi}[1]{doi: #1}\else
  \providecommand{\doi}{doi: \begingroup \urlstyle{rm}\Url}\fi

\bibitem[Ahn et~al.(2007)Ahn, Marron, Muller, and Chi]{ahn2007high}
Jeongyoun Ahn, JS~Marron, Keith~M Muller, and Yueh-Yun Chi.
\newblock The high-dimension, low-sample-size geometric representation holds
  under mild conditions.
\newblock \emph{Biometrika}, 94\penalty0 (3):\penalty0 760--766, 2007.

\bibitem[Bai and Yin(1993)]{bai1993limit}
Zhi-Dong Bai and Yong-Qua Yin.
\newblock Limit of the smallest eigenvalue of a large dimensional sample
  covariance matrix.
\newblock \emph{The Annals of Probability}, 21\penalty0 (3):\penalty0
  1275--1294, 1993.

\bibitem[Baik and Silverstein(2006)]{BAIK20061382}
Jinho Baik and Jack~W. Silverstein.
\newblock Eigenvalues of large sample covariance matrices of spiked population
  models.
\newblock \emph{Journal of Multivariate Analysis}, 97\penalty0 (6):\penalty0
  1382 -- 1408, 2006.
\newblock ISSN 0047-259X.
\newblock \doi{https://doi.org/10.1016/j.jmva.2005.08.003}.
\newblock URL
  \url{http://www.sciencedirect.com/science/article/pii/S0047259X0500134X}.

\bibitem[Bartlett et~al.(2019)Bartlett, Long, Lugosi, and Tsigler]{BLT19}
Peter~L. Bartlett, Philip~M. Long, Gabor Lugosi, and Alexander Tsigler.
\newblock Benign overfitting in linear regression.
\newblock \emph{ArXiv}, abs/1906.11300, 2019.

\bibitem[Belkin et~al.(2018)Belkin, Ma, and Mandal]{belk18}
Mikhail Belkin, Siyuan Ma, and Soumik Mandal.
\newblock To understand deep learning we need to understand kernel learning.
\newblock In Jennifer Dy and Andreas Krause, editors, \emph{Proceedings of the
  35th International Conference on Machine Learning}, volume~80 of
  \emph{Proceedings of Machine Learning Research}, pages 541--549,
  Stockholmsmässan, Stockholm Sweden, 10--15 Jul 2018. PMLR.
\newblock URL \url{http://proceedings.mlr.press/v80/belkin18a.html}.

\bibitem[Belkin et~al.(2019)Belkin, Hsu, and Xu]{Belkin2019TwoMO}
Mikhail Belkin, Daniel Hsu, and Ji~Xu.
\newblock Two models of double descent for weak features.
\newblock \emph{ArXiv}, abs/1903.07571, 2019.

\bibitem[Duchi and Wainwright(2013)]{duchi2013}
John~C. Duchi and Martin~J. Wainwright.
\newblock Distance-based and continuum fano inequalities with applications to
  statistical estimation.
\newblock \emph{ArXiv}, abs/1311.2669, 2013.

\bibitem[Hall et~al.(2005)Hall, Marron, and Neeman]{hall2005geometric}
Peter Hall, James~Stephen Marron, and Amnon Neeman.
\newblock Geometric representation of high dimension, low sample size data.
\newblock \emph{Journal of the Royal Statistical Society: Series B (Statistical
  Methodology)}, 67\penalty0 (3):\penalty0 427--444, 2005.

\bibitem[Hastie et~al.(2019)Hastie, Montanari, Rosset, and Tibshirani]{HMRT19}
Trevor~J. Hastie, Andrea Montanari, Saharon Rosset, and Ryan~J. Tibshirani.
\newblock Surprises in high-dimensional ridgeless least squares interpolation.
\newblock \emph{ArXiv}, abs/1903.08560, 2019.

\bibitem[Johnstone(2001)]{johnstone2001}
Iain~M. Johnstone.
\newblock On the distribution of the largest eigenvalue in principal components
  analysis.
\newblock \emph{The Annals of Statistics}, 29\penalty0 (2):\penalty0 295--327,
  04 2001.
\newblock \doi{10.1214/aos/1009210544}.
\newblock URL \url{https://doi.org/10.1214/aos/1009210544}.

\bibitem[Jung and Marron(2009)]{JM09}
Sungkyu Jung and J~S. Marron.
\newblock Pca consistency in high dimension, low sample size context.
\newblock \emph{Annals of Statistics}, 37, 11 2009.
\newblock \doi{10.1214/09-AOS709}.

\bibitem[Koltchinskii and Lounici(2016)]{KL16}
Vladimir Koltchinskii and Karim Lounici.
\newblock Asymptotics and concentration bounds for bilinear forms of spectral
  projectors of sample covariance.
\newblock In \emph{Annales de l'Institut Henri Poincar{\'e}, Probabilit{\'e}s
  et Statistiques}, volume~52, pages 1976--2013. Institut Henri Poincar{\'e},
  2016.

\bibitem[Penrose(1956)]{penrose1956best}
Roger Penrose.
\newblock On best approximate solutions of linear matrix equations.
\newblock In \emph{Mathematical Proceedings of the Cambridge Philosophical
  Society}, volume~52, pages 17--19. Cambridge University Press, 1956.

\bibitem[Shen et~al.(2013)Shen, Shen, Zhu, and Marron]{shen2013surprising}
Dan Shen, Haipeng Shen, Hongtu Zhu, and JS~Marron.
\newblock Surprising asymptotic conical structure in critical sample
  eigen-directions.
\newblock \emph{arXiv preprint arXiv:1303.6171}, 2013.

\bibitem[Shen et~al.(2016{\natexlab{a}})Shen, Shen, and Marron]{SSM16}
Dan Shen, Haipeng Shen, and J.~S. Marron.
\newblock A general framework for consistency of principal component analysis.
\newblock \emph{Journal of Machine Learning Research}, 17\penalty0
  (150):\penalty0 1--34, 2016{\natexlab{a}}.
\newblock URL \url{http://jmlr.org/papers/v17/14-229.html}.

\bibitem[Shen et~al.(2016{\natexlab{b}})Shen, Shen, Zhu, and Marron]{SSZM16}
Dan Shen, Haipeng Shen, Hongtu Zhu, and J.~S. Marron.
\newblock The statistics and mathematics of high dimension low sample size
  asymptotics.
\newblock \emph{Statistica Sinica}, 26\penalty0 (4):\penalty0 1747--1770,
  2016{\natexlab{b}}.
\newblock ISSN 10170405, 19968507.
\newblock URL \url{http://www.jstor.org/stable/44114356}.

\bibitem[Tengyuan~Liang(2018)]{LR18}
Alexander~Rakhlin Tengyuan~Liang.
\newblock Just interpolate: Kernel "ridgeless" regression can generalize.
\newblock \emph{ArXiv}, abs/1808.00387, 2018.

\bibitem[Vershynin(2010)]{V10}
Roman Vershynin.
\newblock Introduction to the non-asymptotic analysis of random matrices.
\newblock \emph{arXiv preprint arXiv:1011.3027}, 2010.

\bibitem[Wang and Fan(2017)]{wang2017}
Weichen Wang and Jianqing Fan.
\newblock Asymptotics of empirical eigenstructure for high dimensional spiked
  covariance.
\newblock \emph{Ann. Statist.}, 45\penalty0 (3):\penalty0 1342--1374, 06 2017.
\newblock \doi{10.1214/16-AOS1487}.
\newblock URL \url{https://doi.org/10.1214/16-AOS1487}.

\bibitem[Yin et~al.(1988)Yin, Bai, and Krishnaiah]{yin1988limit}
Yong-Quan Yin, Zhi-Dong Bai, and Pathak~R Krishnaiah.
\newblock On the limit of the largest eigenvalue of the large dimensional
  sample covariance matrix.
\newblock \emph{Probability theory and related fields}, 78\penalty0
  (4):\penalty0 509--521, 1988.

\end{thebibliography}
\bibliographystyle{plainnat}

\appendix

\section{Asymptotics of sample covariance matrix in the HDLSS regime}
\label{sec:asympt-sample-covar}

\subsection{Preliminaries}
\label{sec:preliminaries}

Here we investigate the asymptotics of the sample covariance matrix in the HDLSS
regime. Note that this has already been done for instance in
\citet{hall2005geometric,ahn2007high,JM09,shen2013surprising,SSM16,SSZM16} and
we give those results for completeness. Along the way, we extend a bit
the results of \citet{SSM16} under the \cref{ass:iid_comp_strong} to obtain
non-asymptotic bounds in the case where the entries of $\bm{Z}$ are
sub-gaussian.

\medskip%
As shown in \citet{SSM16}, the proofs rely on analyzing the asymptotics of the
dual matrix
$\hat{D} \coloneqq n^{-1}\bm{X}\bm{X}^T = n^{-1}\bm{Z}\Lambda \bm{Z}^T$, which
can be rewritten as
$\hat{D} = n^{-1}\sum_{j=1}^d \lambda_j \tilde{\bm{Z}}_j \tilde{\bm{Z}}_j^T$,
where $\tilde{\bm{Z}}_j \in \Reals^n$ has i.i.d entries
$\tilde{\bm{Z}}_j \coloneqq (Z_{1,j},\dots,Z_{n,j})$. Then, we can decompose
$\hat{D}$ into spiked-part
$\hat{D}_s \coloneqq n^{-1}\sum_{j=1}^{\bar{m}}\lambda_j \tilde{\bm{Z}}_j
\tilde{\bm{Z}}_j^T$ and non-spiked-part
$\hat{D}_{ns} \coloneqq n^{-1}\sum_{j=\bar{m}+1}^d\lambda_j \tilde{\bm{Z}}_j
\tilde{\bm{Z}}_j^T$.

\subsection{On the behaviour of the spiked eigenvalues}
\label{sec:behav-sample-eigenv}

The goal is to demonstrate the following lemma. The first item of the lemma is
taken as it is from \citet[Lemma~3]{SSM16}. The second item is obtained using
the same steps as \citet[Lemma~3]{SSM16} by exploiting the additional structure
offered by \cref{ass:iid_comp_strong}, and the well-known results from
\citet{V10}, which we recall for completeness in
\cref{sec:random-matrix-facts-1}.

\begin{lemma}
	\label{lem:2}
	The following statements are true.%
	\begin{enumerate}
		\item\label{item:1} If \cref{ass:iid_comp} is valid, then for every fixed
		integer $\bar{m}$, and every $1 \leq m \leq \bar{m}$, as $n\to \infty$,
		\begin{equation}
		\label{eq:52}
		\max_{1\leq k \leq m}\frac{|\hat{\lambda}_k -
			\lambda_k|}{\lambda_k}%
		\leq \frac{d\lambda_{\bar{m}+1}}{n\lambda_{m}}\Big\{1 +
		O\Big(\sqrt{\frac{n}{d}} \Big)\Big\}
		+ O\Big( \frac{1}{\sqrt{n}}
		\Big),\quad \textrm{almost-surely}.
		\end{equation}
		
		\item\label{item:2} If \cref{ass:iid_comp_strong} is valid and $0 < t \leq n$, then there is a
		universal constant $C > 0$ such that with $\PP$-probability at least
		$1 - e^{-t}$, for all $1 \leq \bar{m} \leq n$, all $1 \leq m \leq \bar{m}$,
		\begin{equation}
		\label{eq:29}
		\max_{1\leq k \leq m}\frac{|\hat{\lambda}_k -
			\lambda_k|}{\lambda_k}%
		\leq \frac{d\lambda_{\bar{m}+1}}{n\lambda_m}\Big\{1 + C\sqrt{\frac{n \vee
				t}{d}} \Big\} %
		+ C\sqrt{\frac{\bar{m} \vee t}{n}}.
		\end{equation}
	\end{enumerate}
\end{lemma}
\begin{proof}
	We copy \cite[Lemma~3]{SSM16}. Note that \cref{item:1} is simply
	\cite[Lemma~3]{SSM16}, or can be derived using the same steps as
	\cref{item:2}, and thus we only prove the \cref{item:2}. The matrix
	$\hat{\Sigma} \coloneqq n^{-1}\bm{X}^T\bm{X}$ has the same non-zero singular
	values as the dual matrix $\hat{D}$. Then, by Weyl's inequalities, we have for
	any $1 \leq j \leq \bar{m}$ that
	$\sigma_{\min}(\hat{D}_{ns}) + \sigma_j(\hat{D}_{s}) \leq \sigma_j(\hat{D})
	\leq \sigma_j(\hat{D}_s) + \sigma_{\max}(\hat{D}_{ns})$. By Proposition \ref{pro:6}, with
	$\PP$-probability at least $1 - e^{-t}$,
	\begin{equation}
	\label{eq:15}
	\max_{1\leq j \leq \bar{m}} \frac{n}{d \lambda_{\bar{m}
			+1}}| \sigma_j(\hat{D}) - \sigma_j(\hat{D}_s)| \leq 1 + C
	\sqrt{\frac{n\vee t}{d}}.
	\end{equation}
	Hence, it is enough to establish the asymptotics of $\sigma_j(\hat{D}_s)$. We
	proceed as in \cite[Lemma~3]{SSM16} and for $k=1,\dots,\bar{m}$ we introduce the
	matrices
	$\hat{D}_s^k \coloneqq n^{-1}\sum_{j=k}^{\bar{m}}\lambda_j \tilde{\bm{Z}_j}
	\tilde{\bm{Z}}_j^T$. Then, by their equations (14) and (16), for all
	$k=1,\dots,\bar{m}$ it holds
	$\sigma_{\max}( n^{-1}\tilde{\bm{Z}}_k \tilde{\bm{Z}}_k^T) \leq
	\frac{\sigma_k(\hat{D}_s)}{\lambda_k} \leq
	\frac{1}{\lambda_k}\sigma_{\max}(\hat{D}_s^k )$. The result follows from
	Proposition \ref{pro:7}.
\end{proof}

\begin{proposition}
	\label{pro:7}
	Under \cref{ass:iid_comp_strong}, if $0 < t \leq n$ there exists a constant
	$C > 0$ depending on on $\nu$ such that with $\PP$-probability $1 - e^{-t}$,
	for all $1 \leq \bar{m} \leq n$
	\begin{equation}
	\label{eq:16}
	\max_{1\leq k \leq \bar{m}}\frac{\sigma_{\max}(\hat{D}_s^k)}{\lambda_k}%
	\leq 1 + C\sqrt{\frac{\bar{m} \vee t}{n}},\qquad%
	\min_{1\leq k \leq \bar{m}} \sigma_{\max}(n^{-1}\tilde{\bm{Z}}_k
	\tilde{\bm{Z}}_k^T)%
	\geq 1 - C\sqrt{\frac{\bar{m} \vee t}{n}}.
	\end{equation}
\end{proposition}
\begin{proof}
	Those computations are standard. The result for
	$\sigma_{\max}(n^{-1}\tilde{\bm{Z}}_k \tilde{\bm{Z}}_k^T)$ immediately follows
	from Proposition \ref{pro:11} and a union-bound. We proceed with the other bound. As in
	\citet[Lemma~3]{SSM16}, let define
	$W_i \coloneqq
	(\sqrt{\lambda_k}Z_{i,k},\dots,\sqrt{\lambda_{\bar{m}}}Z_{i,\bar{m}}) \in
	\Reals^{\bar{m}-k}$, and $i=1,\dots,n$. Let also $\bm{W}$ be
	$n \times (\bar{m}-k)$ matrix whose rows are $W_i$. Then, remark that
	$\hat{D}_s^k = n^{-1}\bm{W}\bm{W}^T$, and thus $\hat{D}_s^k$ and
	$n^{-1}\bm{W}^T\bm{W}$ have the same non-zero singular values, and it is
	enough to bound in probability $\sigma_{\max}(n^{-1}\bm{W}^T\bm{W})$. First we
	remark that
	$\EE[n^{-1}\bm{W}^T\bm{W}] =
	\mathrm{diag}(\lambda_k,\dots,\lambda_{\bar{m}})$, and thus at least
	$\sigma_{\max}(\EE[n^{-1}\bm{W}^T\bm{W}]) = \lambda_k$. We now show that under
	\cref{ass:iid_comp_strong} there is enough concentration so that the result
	holds in probability. For any $v \in \Reals^{\bar{m}-k}$,
	$n^{-1}\|\bm{W}v\|^2 = v^T(n^{-1}\bm{W}^T\bm{W})v =
	\sum_{j=k}^{\bar{m}}\sum_{\ell=k}^n v_jv_k
	\frac{1}{n}\sum_{i=1}^nW_{i,j}W_{i,\ell}$, \textit{i.e.}
	$n^{-1}\|\bm{W}v\|^2 =
	\sum_{j=k}^{\bar{m}}\sum_{\ell=k}^{\bar{m}}\sqrt{\lambda_j}v_j\sqrt{\lambda_{\ell}}v_{\ell}
	\frac{1}{n}\sum_{i=1}^nZ_{i,j}Z_{i,\ell}$. Indeed, letting the transformation
	$\tau(v) = (\sqrt{\lambda_k}v_k,\dots, \sqrt{\lambda_{\bar{m}}}v_{\bar{m}})$
	and letting $\bm{W}_{*}$ be the $n \times (\bar{m}-k)$ matrix whose rows are
	$(Z_{i,k},\dots,Z_{i,\bar{m}})$, then we can write,
	\begin{equation}
	\label{eq:27}
	\sigma_{\max}(\hat{D}_s^k)%
	= \sigma_{\max}(n^{-1}\bm{W}^T\bm{W})%
	= \sup_{v\in \Reals^{\bar{m}-k}}\frac{n^{-1}\|\bm{W}v\|^2}{\|v\|^2}%
	= \sup_{v\in
		\Reals^{\bar{m}-k}}\frac{n^{-1}\|\bm{W}_{*}\tau(v)\|^2}{\|\tau(v)\|^2}%
	\frac{\|\tau(v)\|^2}{\|v\|^2},
	\end{equation}
	and hence
	$\sigma_{\max}(\hat{D}_s^k) \leq \sigma_{\max}(n^{-1}\bm{W}_{*}^T\bm{W}_{*})
	\sup_{v\in \Reals^{\bar{m}-k}}\frac{\|\tau(v)\|^2}{\|v\|^2} \leq \lambda_k
	\sigma_{\max}(n^{-1}\bm{W}_{*}^T\bm{W}_{*})$. The result then follows from
	Proposition \ref{pro:11} and a union bound.
\end{proof}

\subsection{On the behaviour of non-spiked eigenvalues}
\label{sec:bahviour-non-spiked}

Here we consider the asymptotic behaviour of the non-spiked sample eigenvalues
in the HDLSS regime. In particular, to bound the variance of the MNLS estimator,
we need to understand the behaviour of $\hat{\lambda}_n$. Again, we borrow the
result from \citet{SSM16}.%

\begin{proposition}
	\label{pro:8}
	The following statements are true.
	\begin{enumerate}
		\item Under \cref{ass:iid_comp} it holds
		$\frac{n\hat{\lambda}_n}{d \lambda_d} \geq 1 + O\big(\sqrt{\frac{n}{d}}
		\big)$ almost-surely as $n\to \infty$.
		\item Under \cref{ass:iid_comp_strong}, there exists a universal constant
		$C > 0$ such that with $\PP$-probability at least $1- e^{-t}$ it holds
		$\frac{n\hat{\lambda}_n}{d \lambda_d} \geq 1 - C\sqrt{\frac{n\vee t}{d}}$.
	\end{enumerate}
\end{proposition}
\begin{proof}
	Using Weyl's inequality, we obtain that
	$\hat{\lambda}_n = \sigma_{\min}(\hat{D}) \geq
	\sigma_{\min}(\hat{D}_{ns})$. Then the result is a consequence of the next
	Proposition \ref{pro:6}.
\end{proof}

\begin{proposition}
	\label{pro:6}
	The following statements are true.
	\begin{enumerate}
		\item Under \cref{ass:iid_comp} it holds
		$\frac{n \sigma_{\max}(\hat{D}_{ns})}{d \lambda_{\bar{m}+1}} \leq 1 +
		O\big(\sqrt{\frac{n}{d}}\big)$ almost-surely as $n\to \infty$, and
		$\frac{n \sigma_{\min}(\hat{D}_{ns})}{d \lambda_d} \geq 1 +
		O\big(\sqrt{\frac{n}{d}}\big)$ almost-surely as $n\to\infty$.
		\item Under \cref{ass:iid_comp_strong}, there exists a universal constant $C
		> 0$ such that with $\PP$-probability at least $1 - e^{-t}$ it holds
		$\frac{n \sigma_{\max}(\hat{D}_{ns})}{d \lambda_{\bar{m}+1}} \leq 1 +
		C\sqrt{\frac{n\vee t}{d}}$, and
		$\frac{n \sigma_{\min}(\hat{D}_{ns})}{d \lambda_d} \geq 1 -
		C\sqrt{\frac{n\vee t}{d}}$.
	\end{enumerate}
\end{proposition}
\begin{proof}
	We copy \cite[Lemma~4]{SSM16}. It is enough to consider
	$\hat{D}_{ns}^{*} \coloneqq \frac{1}{n}\sum_{j=\bar{m}+1}^d \tilde{\bm{Z}}_j
	\tilde{\bm{Z}}_j^T$. Indeed by the equation (22) in their paper for every
	$j \geq 1$ it holds
	$\lambda_d \sigma_j(\hat{D}_{ns}^{*}) \leq \sigma_j(\hat{D}_{ns}) \leq
	\lambda_{\bar{m}+1} \sigma_j(\hat{D}_{ns}^{*})$. Then the result follows by
	Propositions \ref{pro:10}, and \ref{pro:11}, because $\hat{D}_{ns}^{*}$ is a $(d-\bar{m})\times n$
	matrix with i.i.d entries of zero mean and finite variance.
\end{proof}

\subsection{On the behaviour of the eigen-projectors }
\label{sec:behav-eigen-proj}

Here we mostly follow the results in \citet{KL16} instead of \citet{SSM16},
which provides a simpler approach to bounding $\|P_j - \hat{P}_j\|$. In
particular, the following proposition is a restatement of their more general
\citet[Lemma~1]{KL16}. Then, the main result of this section, given in
Lemma \ref{lem:3}, simply follows from the next proposition and classical random
matrix theory arguments.

\begin{proposition}[\citet{KL16}]
	\label{pro:9}
	Let $\Sigma = \sum_{j=1}^d \lambda_j P_j$, where $(\lambda_1,\dots,\lambda_d)$
	are the eigenvalues of $\Sigma$ sorting in decreasing order, \textit{i.e.}
	$\lambda_1 \geq \dots \geq \lambda_d$ and $P_j$ is the projection operator
	onto the span of the $j$-th eigenvector of $\Sigma$. Similarly, let
	$\tilde{\Sigma} = \sum_{j=1}^d \tilde{\lambda}_j \tilde{P}_j$. Also let
	$g_j \coloneqq \lambda_j - \lambda_{j+1}$ denote the $j$-th spectral gap of
	$\Sigma$, $\bar{g}_1 \coloneqq g_1$, and
	$\bar{g}_j \coloneqq \min\{g_{j-1},g_j\}$ for $j=2,\dots,d$. Then,
	$\|\tilde{P}_j - P_j\| \leq \frac{4\|\tilde{\Sigma} - \Sigma\|}{\bar{g}_j}$.
\end{proposition}

\begin{lemma}
	\label{lem:3}
	The following statements are true.%
	\begin{enumerate}
		\item\label{item:ev:1} If \cref{ass:iid_comp} is valid, then for every fixed
		integer $m$, as $n\to \infty$,
		\begin{equation}
		\label{eq:49}
		\max_{1\leq j \leq m} \frac{\bar{G}_j\|\hat{P}_j - P_j\|}{\lambda_1}%
		= O\Big(\sqrt{\frac{1}{n}} \bigvee
		\sqrt{\frac{d\lambda_{m+1}}{n\lambda_1}} \bigvee 
		\frac{d\lambda_{m+1}}{n\lambda_1}  \Big),\quad \textrm{almost-surely}.
		\end{equation}
		
		\item\label{item:ev:2} If \cref{ass:iid_comp_strong} is valid then there is
		a universal constant $C > 0$ such that with probability at least $1-e^{-t}$,
		for all $m =1,\dots,n$,
		\begin{equation}
		\label{eq:50}
		\max_{1\leq j \leq m} \frac{\bar{G}_j\|\hat{P}_j - P_j\|}{C\lambda_1}%
		\leq \sqrt{\frac{m \vee t}{n}} \bigvee
		\sqrt{\frac{d\lambda_{m+1}}{n\lambda_1}} \bigvee
		\frac{d\lambda_{m+1}}{n\lambda_1}.
		\end{equation}
	\end{enumerate}
\end{lemma}
\begin{proof}
	In order to apply Proposition \ref{pro:9}, we need to figure out an upper bound for
	$\|\hat{\Sigma} - \Sigma\|$. Remark that $\Sigma = U\Lambda U^T$ and
	$\hat{\Sigma} = n^{-1}\bm{X}^T\bm{X} = n^{-1}U\Lambda^{1/2}\bm{Z}^T\bm{Z}
	\Lambda^{1/2} U^T$. Then, we have the following chain of estimates, as $\|U\|=\|U^T\|=1$
	\begin{align}
	\label{eq:30}
	\|\hat{\Sigma} - \Sigma\|%
	=\|U( n^{-1} \Lambda^{1/2}\bm{Z}^T\bm{Z}\Lambda^{1/2} - \Lambda ) U^T \|%
	= \Big\|\Lambda^{1/2}\Big(\frac{\bm{Z}^T\bm{Z}}{n} - I\Big)\Lambda^{1/2}\Big\|.
	\end{align}
	We split the space $\Reals^d$ onto two orthogonal subspaces, corresponding to
	projection on $S\coloneqq \mathrm{span}(e_1,\dots,e_{m})$ and
	$S_{\bot} = \mathrm{span}(e_{m+1},\dots,e_d)$. Then, we let $\Lambda_S$,
	respectively $\bm{Z}_S$, denote the restriction of $\Lambda$ to $S$,
	respectively $\bm{Z}$. Similarly we let $\Lambda_{\bot}$, respectively
	$\bm{Z}_{\bot}$ the restrictions to $S_{\bot}$. Then we rewrite by blocks,
	\begin{equation}
	\label{eq:31}
	\Lambda^{1/2}\Big(\frac{\bm{Z}^T\bm{Z}}{n} - I\Big)\Lambda^{1/2}%
	=%
	\begin{pmatrix}
	\Lambda_S^{1/2}(n^{-1}\bm{Z}_S^T\bm{Z}_S - I)\Lambda_S^{1/2} &
	n^{-1}\Lambda_{\bot}^{1/2}\bm{Z}_{\bot}^T\bm{Z}_S\Lambda_S^{1/2}\\
	n^{-1}\Lambda_S^{1/2}\bm{Z}_S^T\bm{Z}_{\bot}\Lambda_{\bot}^{1/2} &
	\Lambda_{\bot}^{1/2}(n^{-1}\bm{Z}_{\bot}^T\bm{Z}_{\bot} - I)\Lambda_{\bot}^{1/2}
	\end{pmatrix}.
	\end{equation}
	Hence, combining the expressions \cref{eq:30,eq:31}, we can bound
	$\|\hat{\Sigma} - \Sigma\|$ as
	\begin{multline}
	\label{eq:46}
	\|\hat{\Sigma} - \Sigma\|%
	\leq \Big\|\Lambda_S^{1/2}\Big(\frac{\bm{Z}_S^T\bm{Z}_S}{n} -
	I\Big)\Lambda_S^{1/2} \Big\|%
	+ \Big\|\Lambda_{\bot}^{1/2}\frac{\bm{Z}_{\bot}^T\bm{Z}_S}{n}\Lambda_S^{1/2}
	\Big\|\\%
	+ \Big\|\Lambda_S^{1/2}\frac{\bm{Z}_S^T\bm{Z}_{\bot}}{n}\Lambda_\bot^{1/2}
	\Big\|%
	+ \Big\|\Lambda_\bot^{1/2}\Big(\frac{\bm{Z}_\bot^T\bm{Z}_\bot}{n} -
	I\Big)\Lambda_\bot^{1/2} \Big\|.
	\end{multline}
	The rhs of the last display is bounded by,
	\begin{equation}
	\label{eq:47}
	\Big\|\Lambda_S^{1/2}\Big(\frac{\bm{Z}_S^T\bm{Z}_S}{n} -
	I\Big)\Lambda_S^{1/2} \Big\|%
	+ 2\Big\|\frac{\bm{Z}_S}{\sqrt{n}}\Lambda_S^{1/2} \Big\|
	\Big\|\frac{\bm{Z}_{\bot}}{\sqrt{n}}\Lambda_{\bot}^{1/2}  \Big\|%
	+
	\Big\|\Lambda_{\bot}^{1/2}\frac{\bm{Z}_{\bot}^T\bm{Z}_{\bot}}{n}\Lambda_{\bot}^{1/2}\Big\|
	+ \|\Lambda_{\bot}\|,
	\end{equation}
	which is in turn bounded by
	\begin{equation}
	\label{eq:32}
	\|\Lambda_S^{1/2}\|^2\Big\|\frac{\bm{Z}_S^T\bm{Z}_S}{n} - I \Big\|%
	+ 2 \|\Lambda_S^{1/2}\|\|\Lambda_{\bot}^{1/2}\|
	\Big\|\frac{\bm{Z}_S}{\sqrt{n}}\Big\| \Big\|\frac{\bm{Z}_{\bot}}{\sqrt{n}}
	\Big\| + \|\Lambda_{\bot}^{1/2}\|^2
	\Big\|\frac{\bm{Z}_{\bot}^T\bm{Z}_{\bot}}{n} \Big\| + \|\Lambda_{\bot}\|.
	\end{equation}
	Since $\|\Lambda_S^{1/2}\| = \sqrt{\lambda_1}$, $\|\lambda_{\bot}^{1/2}\| =
	\sqrt{\lambda_{m+1}}$ and $\|\lambda_{\bot}\| = \lambda_{m+1}$, we deduce that
	\begin{equation}
	\label{eq:33}
	\|\hat{\Sigma} - \Sigma\|%
	\leq \lambda_1 \Big\| \frac{\bm{Z}_S^T\bm{Z}_S}{n} - I \Big\|%
	+ 2 \sqrt{\lambda_1
		\lambda_{m+1}}\sigma_{\max}\Big(\frac{\bm{Z}_S}{\sqrt{n}}\Big)%
	\sigma_{\max}\Big(\frac{\bm{Z}_{\bot}}{\sqrt{n}} \Big)%
	+ \lambda_{m+1}\sigma_{\max}\Big(\frac{\bm{Z}_{\bot}}{\sqrt{n}} \Big)^2%
	+ \lambda_{m+1}.
	\end{equation}
	We now consider only \cref{item:ev:2}. On the event that
	$\sigma_{\max}(n^{-1/2}\bm{Z}_S) \leq 1 + C\sqrt{m/n} + \sqrt{t/c}$ and
	$\sigma_{\min}(n^{-1/2}\bm{Z}_S) \geq 1 - C\sqrt{m/n} - \sqrt{t/c}$, it
	is easily seen that
	$\|n^{-1}\bm{Z}_S^T\bm{Z}_S - I\| \leq C\sqrt{m/n} + \sqrt{t/c}$; see
	for instance \citet[Lemma~5.36]{V10}. Further, by Proposition \ref{pro:11} this event has
	probability at least $1-4\exp(- t)$ for appropriate choice of $C,c > 0$. The
	other terms are also bounded using Proposition \ref{pro:11}. In fact, using that
	$0 < t \leq n$ and $d \geq n$, we can show that there exists a constant
	$K > 0$ depending only on $\nu$ such that with probability at least $1-e^{-t}$
	(by eventually increasing the constants if needed),
	\begin{equation}
	\label{eq:22}
	\|\hat{\Sigma} - \Sigma\| \leq K\lambda_1\Big(\sqrt{\frac{m}{n}} \bigvee
	\sqrt{\frac{t}{n}} \bigvee%
	\sqrt{\frac{d\lambda_{m+1}}{n\lambda_1}} \bigvee \frac{d
		\lambda_{m+1}}{n\lambda_1}\Big).
	\end{equation}
	The result follows by combining the last display with Proposition \ref{pro:9}. The proof
	for \cref{item:ev:1} is similar but uses Proposition \ref{pro:10} instead of
	Proposition \ref{pro:11}, or could be derived from the results in \citet{SSM16}.
\end{proof}

\subsection{Bias variance decomposition  of risk}
\label{decomposition_expanded}
To show \eqref{decomposition}, we start 
with adding and subtracting 
$x_{new}^T \E [\hat{\theta} \mid X]$ 
\begin{align}
\E [(x_{new}^T \theta_* - x_{new}^T \hat{\theta})^2 \mid X] %
=
\E [(x_{new}^T \theta_* - x_{new}^T \E [\hat{\theta} \mid X])^2 \mid X] + 
\E [(x_{new}^T \E [\hat{\theta} \mid X] - x_{new}^T \hat{\theta})^2 \mid X],
\end{align}
note that expectation of cross term is zero. 

Since 
\begin{equation}
\E [\hat{\theta} \mid X] = 
\E [(X^TX)^\dagger XY \mid X] =
(X^TX)^\dagger X^TX \theta_*,
\end{equation}
the first term which is bias becomes
\begin{align}
\E [(x_{new}^T \theta_* - x_{new}^T \E [\hat{\theta} \mid X])^2 \mid X]
=
\E [(x_{new}^T (\theta_* - (X^TX)^\dagger X^TX
\theta_*))^2 \mid X]
\\= 
\E [(x_{new}^T (I - (X^TX)^\dagger X^TX
)\theta_*)^2 \mid X].
\end{align}
Next we expand the square above. By definition
\begin{align}
\E [ x_{new} x_{new}^T \mid X] = \Sigma 
\end{align}
and $ \frac{1}{n} X^T X = \hat{\Sigma}$.
So we can write bias
\begin{align}
\E [(x_{new}^T \theta_* - x_{new}^T \E [\hat{\theta} \mid X])^2 \mid X]
=
\Tr(
\theta_*^T (I - \hat{\Sigma}^\dagger \hat{\Sigma}) 
\Sigma
(I - \hat{\Sigma}^\dagger \hat{\Sigma}) \theta_*
).
\end{align}
Similarly for variance we get
\begin{align}
\E [(x_{new}^T \E [\hat{\theta} \mid X] - x_{new}^T \hat{\theta})^2 \mid X] 
=
\E [ (x_{new}^T ((X^TX)^\dagger X^T (X \theta_* -  Y)))^2 \mid X]
\\=
\E [ (x_{new}^T ((X^TX)^\dagger X^T \xi))^2 \mid X]
= 
\E [x_{new}^T (X^TX)^\dagger X^T \xi 
\xi^T X (X^T X)^\dagger x_{new} \mid X]
\\ \le 
\sigma ^2 
\E [x_{new}^T (X^TX)^\dagger X^T
X (X^T X)^\dagger x_{new} \mid X]
=
sigma ^2 
\E [x_{new}^T (X^TX)^\dagger x_{new} \mid X]
= 
\frac {\sigma^2}{n} \Tr(
\hat{\Sigma}^\dagger
\Sigma
),
\end{align}
where the second to last equality follows from 
definition of pseudo-inverse.

\subsection{Random matrix facts}
\label{sec:random-matrix-facts-1}

The following useful proposition combines famous results from
\citet{yin1988limit,bai1993limit} about the asymptotic behavior or large
covariance matrices, see also \citet[Theorem~2.1]{V10}.

\begin{proposition}[Bai-Yin's law]
	\label{pro:10}
	Let $\bm{W}$ be a $n \times p$, $n > p$, matrix with i.i.d entries $W_{i,j}$
	such that $\EE[W_{i,j}] = 0$, $\EE[W_{i,j}^2] = 1$, and
	$\EE[W_{i,j}^4]<\infty$. Let $y = \lim_{n\to \infty} p/n$. Then
	$\lim_{n\to\infty} \sigma_{\max}(n^{-1/2}\bm{W}) = 1 + \sqrt{y}$ and
	$\lim_{n\to \infty} \sigma_{\min}(n^{-1/2}\bm{W})= 1 - \sqrt{y}$
	almost-surely.
\end{proposition}

The following proposition is copied from \citet[Theorem~5.39]{V10}.

\begin{proposition}
	\label{pro:11}
	Let $\bm{W}$ be a $n \times p$, $n > p$, matrix with i.i.d entries $W_{i,j}$ such that
	$\EE[W_{i,j}] = 0$, $\EE[W_{i,j}^2] = 1$ and there exists $\nu > 0$ such that
	$\log \EE[e^{\lambda W_{i,j} }] \leq \frac{\lambda^2\nu}{2}$ for all
	$\lambda \in \Reals$. Then there are constants $C,c > 0$ depending only on
	$\nu$ such that with probability at least $1 - 2\exp(-ct^2)$ one has
	\begin{equation}
	\label{eq:21}
	\sqrt{n} - C\sqrt{p} - t \leq \sigma_{\min}(\bm{W}) \leq
	\sigma_{\max}(\bm{W}) \leq \sqrt{n} + C\sqrt{p} + t.
	\end{equation}
\end{proposition}

\end{document}